\documentclass[twoside,11pt]{article}

\usepackage{jmlr2e}
\usepackage{amsmath,amsfonts,amssymb}
\usepackage{xcolor}
\usepackage{enumerate}
\usepackage{bm}
\usepackage{dsfont}
\usepackage{verbatim}
\usepackage{multirow}
\usepackage{threeparttable}
\usepackage{subfigure}

\usepackage[noend]{algpseudocode}
\usepackage{algorithmicx,algorithm}

\newcommand{\A}{\mathbf{A}}
\newcommand{\B}{\mathbf{B}}
\newcommand{\T}{\mathbf{T}}
\newcommand{\C}{\mathbf{C}}
\newcommand{\K}{\mathbf{K}}

\newcommand{\U}{\mathcal{U}}
\newcommand{\OT}{\operatorname{OT}}
\newcommand{\UOT}{\operatorname{UOT}}

\newcommand{\KL}{\operatorname{KL}}
\newcommand{\tC}{\widetilde{\mathbf{C}}}
\newcommand{\tK}{\widetilde{\mathbf{K}}}
\newcommand{\tT}{\widetilde{\mathbf{T}}}
\newcommand{\tA}{\widetilde{\mathbf{A}}}
\newcommand{\tO}{\widetilde{O}}
\newcommand{\one}{{\mathbf{1}_n}}
\def\u{\bm{u}}
\def\v{\bm{v}}
\def\a{\bm{a}}
\def\b{\bm{b}}
\def\x{\bm{x}}
\def\y{\bm{y}}
\def\eps{\varepsilon}
\def\lam{\lambda}
\newcommand{\nnz}{\mbox{nnz}}
\newcommand{\diag}{\operatorname{diag}}
\newcommand{\RR}{\mathbb{R}}

\usepackage{lastpage}
\jmlrheading{24}{2023}{1-\pageref{LastPage}}{11/22; Revised 4/23}{6/23}{22-1311}{Mengyu Li, Jun Yu, Tao Li, and Cheng Meng}

\ShortHeadings{Importance Sparsification for Sinkhorn Algorithm}{Li, Yu, Li, and Meng}
\firstpageno{1}

\begin{document}

\title{Importance Sparsification for Sinkhorn Algorithm}

\author{\name Mengyu Li \email limengyu516@ruc.edu.cn \\
       \addr Institute of Statistics and Big Data\\
       Renmin University of China\\
       Beijing, China
       \AND
       \name Jun Yu$^{*}$ \email yujunbeta@bit.edu.cn \\
       \addr School of Mathematics and Statistics\\
       Beijing Institute of Technology\\
       Beijing, China\\
       $^*$Joint first author
       \AND
       \name Tao Li \email 2019000153lt@ruc.edu.cn \\
       \addr Institute of Statistics and Big Data\\
       Renmin University of China\\
       Beijing, China
       \AND
       \name Cheng Meng$^{\dag}$ \email chengmeng@ruc.edu.cn \\
       \addr Center for Applied Statistics, Institute of Statistics and Big Data\\
       Renmin University of China\\
       Beijing, China\\
       $^\dag$Corresponding author}
\editor{Michael Mahoney}

\maketitle

\begin{abstract}
Sinkhorn algorithm has been used pervasively to approximate the solution to optimal transport (OT) and unbalanced optimal transport (UOT) problems. However, its practical application is limited due to the high computational complexity. To alleviate the computational burden, we propose a novel importance sparsification method, called \textsc{Spar-Sink}, to efficiently approximate entropy-regularized OT and UOT solutions. Specifically, our method employs natural upper bounds for unknown optimal transport plans to establish effective sampling probabilities, and constructs a sparse kernel matrix to accelerate Sinkhorn iterations, reducing the computational cost of each iteration from $O(n^2)$ to $\widetilde{O}(n)$ for a sample of size $n$. Theoretically, we show the proposed estimators for the regularized OT and UOT problems are consistent under mild regularity conditions. Experiments on various synthetic data demonstrate \textsc{Spar-Sink} outperforms mainstream competitors in terms of both estimation error and speed. A real-world echocardiogram data analysis shows \textsc{Spar-Sink} can effectively estimate and visualize cardiac cycles, from which one can identify heart failure and arrhythmia. To evaluate the numerical accuracy of cardiac cycle prediction, we consider the task of predicting the end-systole time point using the end-diastole one. Results show \textsc{Spar-Sink} performs as well as the classical Sinkhorn algorithm, requiring significantly less computational time.
\end{abstract}


\section{Introduction}\label{sec:intro}

The optimal transport (OT) problem, initiated by Gaspard Monge in the 18th century, aims to calculate the Wasserstein distance that quantifies the discrepancy between two probability measures.
Recently, the Wasserstein distance has played an increasingly preponderant role in machine learning \citep{courty2016optimal, arjovsky2017wasserstein, meng2019large, muzellec2020missing, balaji2020robust}, statistics \citep{flamary2018wasserstein, panaretos2019statistical, meng2020sufficient, dubey2020functional}, computer vision \citep{ferradans2014regularized, su2015optimal, solomon2015convolutional, xu2019learning}, biomedical research \citep{tanay2017scaling, schiebinger2019optimal, marouf2020realistic}, among others. 
We refer to \cite{peyre2019computational} and \cite{panaretos2019statistical} for recent reviews.

Despite the broad range of applications, existing methods for computing the Wasserstein distance suffer from a huge computational burden when the sample size $n$ is large.
Specifically, traditional approaches involve solving differential equations \citep{brenier1997homogenized, benamou2002monge} or linear programming problems \citep{rubner1997earth, pele2009fast}.
The computational cost of such methods is of the order $O(n^3\log(n))$.

To alleviate the computational burden, a large number of efficient computational tools have been developed in the recent decade.
One major class of approaches is called the regularization-based method, which solves an entropy-regularized OT problem instead of the original one \citep{cuturi2013sinkhorn}. 
The regularized OT problem is unconstrained and convex with a differentiable objective function, and can be solved using the Sinkhorn algorithm \citep{sinkhorn1967concerning} in $O(Ln^2)$ time, where $L$ is the number of iterations. 
It has been shown that regularized OT solutions possess better theoretical properties than the unregularized counterparts \citep{montavon2016wasserstein, rigollet2018entropic, feydy2019interpolating, peyre2019computational}.
Another advantage of the regularization-based approach is that it can be applied to a generic class of unbalanced optimal transport (UOT) problems \citep{chizat2018scaling}. 
The UOT problem relaxes the strict marginal constraints of OT by allowing partial displacement of mass, making it more suitable for applications that involve both mass variation (e.g., creation or destruction) and mass transportation \citep{frogner2015learning, chizat2018scaling, zhou2018wasserstein, wang2020robust}.
The Sinkhorn algorithm can be naturally extended to approximate UOT solutions, also requiring an $O(Ln^2)$ computational cost \citep{chizat2018scaling,pham2020unbalanced}.
In general, the Sinkhorn algorithm enables researchers to approximate the OT and UOT solutions efficiently, and thus has been extensively studied in the recent decade \citep{cuturi2014fast, genevay2019sample, feydy2019interpolating, lin2019efficient, pham2020unbalanced}.
There also exist slicing-based methods to approximate the Wasserstein distance \citep{pitie2005n, rabin2011wasserstein, bonneel2015sliced, meng2019large, deshpande2019max, zhang2021review, nguyen2021distributional, nguyen2023hierarchical}, and such methods are beyond the scope of this paper. A recent review of such methods can be found in \cite{nadjahi2021sliced}.

Despite the wide application, the time and memory requirements of the Sinkhorn algorithm grow quadratically with $n$, which hinders its broad applicability to many large-scale optimal transport problems.
To address the computational bottleneck, many efficient variants of the Sinkhorn algorithm have been proposed in recent years \citep{solomon2015convolutional, altschuler2019massively, pham2020unbalanced, scetbon2020linear, scetbon2021low, klicpera2021scalable, le2021robust, sejourne2022faster, liao2022fast, liao2022fast2}.
For example, in contrast to the scheme of Sinkhorn that updates all rows and columns of the transport plan at each step, the variants including greedy Sinkhorn (\textsc{Greenkhorn}) \citep{altschuler2017near, lin2022efficiency}, randomized Sinkhorn (\textsc{Randkhorn}) \citep{lin2019acceleration}, and screening Sinkhorn (\textsc{Screenkhorn}) \citep{alaya2019screening} only update partial row(s) or column(s) in each iteration, based on different selection criteria. These variants have been shown to converge faster in practice, making them appealing for large-scale applications.
In addition, \cite{xie2020fast} developed an inexact proximal point method to address numerical instability issues of Sinkhorn algorithm.

Nevertheless, most of the existing variants of the Sinkhorn algorithm still require an $O(Ln^2)$ computational cost.
One exception is the \textsc{Nys-Sink} approach proposed by \cite{altschuler2019massively}, where the authors proposed to accelerate the Sinkhorn algorithm using the Nystr{\"o}m method, a well-known technique for low-rank matrix approximation \citep{kumar2012sampling}.
The computational complexity of \textsc{Nys-Sink} is reduced to $O(Lrn)$, where $r\leq n$ denotes the estimated rank of the kernel matrix $\K$ with respect to (w.r.t.) the Sinkhorn algorithm.
Further details of the kernel matrix $\K$ will be provided in the subsequent section.
However, the \textsc{Nys-Sink} method suffers from two limitations: it requires (i) $\K$ to be symmetric positive semi-definite, and (ii) $\K$ possessing a low-rank structure.
Such constraints restrict the applicability of \textsc{Nys-Sink} in many practical scenarios.
For instance, the Wasserstein-Fisher-Rao distance \citep{kondratyev2016new, chizat2018interpolating, liero2018optimal}, a popular distance in UOT problems, is associated with a kernel matrix $\K$ that is highly sparse and nearly full-rank; see Section~\ref{sec:background} for more details.
The \textsc{Nys-Sink} method thus may be ineffective for estimating the Wasserstein-Fisher-Rao distance in large-scale UOT problems.
Therefore, the development of an efficient variant of the Sinkhorn algorithm capable of handling large-scale asymmetric and nearly full-rank kernel matrix $\K$ remains a blank field requiring further research.

In this paper, we propose a randomized sparsification variant of the Sinkhorn algorithm, called \textsc{Spar-Sink}, for both OT and UOT problems.
Specifically, we construct a sparsified kernel matrix $\tK$ by carefully sampling $s=o(n^2)$ elements from $\K$ and setting the remaining ones to zero.
We then leverage $\tK$ and sparse matrix multiplications to accelerate the iterations in the Sinkhorn algorithm, reducing the computational cost from $O(n^2)$ to $O(s)$ per iteration.

The key to the success of the proposed strategy is developing an effective sampling probability.
We demonstrate that both OT and UOT problems provide natural upper bounds for the elements in the unknown optimal transport plan.
Drawing inspiration from the importance sampling technique, we employ such upper bounds to construct sampling probabilities.
Theoretically, we show that the proposed estimators for entropic OT and UOT problems are consistent when $s=\widetilde{O}(n)$ under certain regularity conditions, where $\widetilde{O}(\cdot)$ suppresses logarithmic factors.
Extensive simulations show \textsc{Spar-Sink} yields much smaller estimation errors compared with mainstream competitors.

We consider a real-world echocardiogram data analysis to demonstrate the performance of \textsc{Spar-Sink}.
Specifically, we propose using the Wasserstein-Fisher-Rao (WFR) distance \citep{kondratyev2016new,chizat2018interpolating,liero2018optimal}, a special metric in UOT problems, to characterize the similarity between any two frames in an echocardiogram video. 
Compared to the Wasserstein distance, the WFR distance prevents long-range mass transportation between two distributions, and thus can achieve a balance between global transportation and local truncation.
Intuitively, such a distance is more consistent with the nature of myocardial motion that the cardiac muscle would not transport too far.
We focus on the task of cardiac cycle identification, which is an essential but laborious task for the downstream assessment of cardiac function \citep{ouyang2020video}.
We apply the proposed \textsc{Spar-Sink} algorithm to approximate the WFR distance and predict cardiac cycles automatically and efficiently, which has the potential to obviate the heavy work for cardiologists.
Empirical results show that our method can effectively estimate and visualize cardiac cycles, with the potential to identify heart failure and arrhythmia from the results.
To evaluate the numerical accuracy of cardiac cycle prediction, we predict the end-systole time point using the end-diastole one.
The results show \textsc{Spar-Sink} achieves the same prediction accuracy as the Sinkhorn algorithm while requiring much less computational time. 

A problem closely related to the optimal transport is the (fixed-support) Wasserstein barycenter problem, which aims to calculate the barycenter of a set of probability measures (whose supports are predetermined) in the Wasserstein space \citep{agueh2011barycenters}. 
Extending the work of \cite{cuturi2013sinkhorn}, Wasserstein barycenters can also be approximated by entropic smoothing \citep{cuturi2014fast} using the iterative Bregman projection (IBP) algorithm \citep{benamou2015iterative}. 
Concerning the computational hardness, significant research has been devoted to further enhancing the celebrated IBP algorithm \citep{cuturi2018semidual, kroshnin19complexity, lin2020fixed, guminov2021combination}. 
In this paper, we also extend the idea of sparsification to the IBP algorithm, efficiently approximating Wasserstein barycenters.

The remainder of this paper is organized as follows. We start in Section~\ref{sec:background} by introducing the background of OT and UOT problems. 
In Section~\ref{sec:method}, we develop the sampling probabilities and provide the details of the main algorithm.
The theoretical properties of the proposed estimators are presented in Section~\ref{sec:theory}.
We examine the performance of the proposed method through extensive synthetic data sets in Section~\ref{sec:simu}.
Echocardiogram data analysis is provided in Section~\ref{sec:real}.
Extensions, technical details, and additional numerical results and applications are relegated to the Appendix.

\section{Background}\label{sec:background}
Here we summarize the notation used throughout the paper. We adopt the standard convention of using uppercase boldface letters for matrices, lowercase boldface letters for vectors, and regular font for scalars. 
We denote non-negative real numbers by $\mathbb{R}_{+}$, the set of integers $\{1, \ldots, n\}$ by $[n]$, and the $(n-1)$-dimensional simplex by $\Delta^{n-1}=\{\x \in \mathbb{R}_{+}^{n}: \sum_{i=1}^{n} x_{i}=1\}$. 
An empirical measure $\mu$ supported by $n$ points $\x_i \in \mathbb{R}^{d}, i\in [n]$ is defined as $\mu=\sum_{i=1}^{n} a_{i} \delta_{\x_{i}}$, where $\delta_\cdot$ is the Dirac delta function and $\a = (a_1, \ldots, a_n)$ is the corresponding histogram in $\mathbb{R}_{+}^{n}$.
For two histograms $\a, \b \in \mathbb{R}_{+}^{n}$, we define the Kullback-Leibler divergence $\mathrm{KL}(\a \| \b)$ between $\a$ and $\b$ by $\mathrm{KL}(\a \| \b)=\sum_{i=1}^{n} a_{i} \log (a_{i}/b_{i})-a_{i}+b_{i}$, where we adopt the standard convention that $0 \log(0)=0$.
For a coupling matrix $\T \in \mathbb{R}_{+}^{n \times n}$, its Shannon entropy is defined as $H(\T)=-\sum_{i,j} T_{ij} (\log (T_{ij})-1)$. 
We use $\|\A\|_2$ to denote the spectral norm (i.e., maximal singular value) of a matrix $\A$, and its condition number is defined as $\|\A\|_2/\lambda_{\min}(\A)$, where $\lambda_{\min}(\cdot)$ is the minimal singular value. For $\A$ and $\B$ of the same dimension, we denote their Frobenius inner product by $\langle \A, \B \rangle=\sum_{i,j} A_{i j} B_{i j}$. 
For a vector $\bm x$, we use $\|\bm x\|_p$ and $\|\bm x\|_\infty$ to represent its $\ell_p$ norm and infinity norm, respectively.
For two non-negative sequences $(x_{n})_{n}$ and $(y_{n})_{n}$, we denote $x_{n}=\widetilde{O}(y_{n})$ if there exist constants $c, c^\prime>0$ such that $x_{n} \leq c^\prime y_{n}(\log (n))^{c}$.

\subsection{Optimal Transport Problem and Sinkhorn Algorithm}
To begin with, we consider two empirical probability measures $\a \in \Delta^{m-1}$ and $\b \in \Delta^{n-1}$. For brevity, we focus on the case of $m=n$ in this paper, since the extension to unequal cases is straightforward.
The goal of the optimal transport problem is to compute the minimal effort of moving the masses $\a$ and $\b$ onto each other, according to some ground cost between the supports. Due to \citet{kantorovich1942transfer}, the modern OT formulation takes the form
\begin{equation}\label{eq:ot}
\OT(\a, \b):=\min _{\T \in \mathcal{U}(\a, \b)}\langle \T, \C\rangle,
\end{equation}
where $\mathcal{U}(\a, \b):= \{\T \in \mathbb{R}_{+}^{n \times n}: \T \mathbf{1}_{n} = \a, \T^{\top} \mathbf{1}_{n} = \b\}$ is the set of admissible transportation plans, i.e., all joint probability distributions with marginals $\a, \b$, and $\C \in \mathbb{R}_{+}^{n \times n}$ is a given cost matrix with bounded entries. 
The solutions to~\eqref{eq:ot} are called the optimal transport plan.
When $\C$ is a pairwise distance matrix of the power $p$, $W_p(\cdot, \cdot) := \OT(\cdot, \cdot)^{1/p}$ defines the $p$-Wasserstein distance on $\Delta^{n-1}$.

A clear drawback of OT is that the computational cost of directly solving the problem~\eqref{eq:ot} is hugely prohibitive. 
Indeed, conventional methods require $O(n^3 \log(n))$ time \citep{brenier1997homogenized, rubner1997earth, benamou2002monge, pele2009fast}. Even the fastest algorithms known to date for~\eqref{eq:ot} have a computational complexity of at least $O(n^{2.5} \log(n))$ \citep{lee2014path, lee2015efficient, guo2020fast, an2022efficient}.

To approximate the solution to the OT problem efficiently, \cite{cuturi2013sinkhorn} introduced an entropic penalty term to \eqref{eq:ot} and turned it into an entropy-regularized OT problem
\begin{equation}\label{eq:rot}
\OT_\eps(\a, \b) := \min _{\T \in \mathcal{U}(\a, \b)} \langle \T, \C\rangle - \eps H(\T),
\end{equation}
where the regularization parameter $\eps>0$ controls the strength of the penalty term. 
Let $\T_\eps^\ast \in \mathbb{R}^{n \times n}$ be the solution to~\eqref{eq:rot}. 
It is known that when $\eps \to 0$, $\OT_\eps(\a, \b)\to\OT(\a, \b)$; when $\eps \to \infty$, $\T_\eps^\ast \to \a\b^\top$ \citep{peyre2019computational}.

In general, the solution $\T_\eps^\ast$ is a projection onto $\mathcal{U}(\a, \b)$ of the kernel matrix $\K := \exp(-\C / \eps)$. The $(i,j)$th entry of $\K$ is given by $K_{ij} = \exp(-C_{ij} / \eps)$.
Indeed, for two (unknown) convergent scaling vectors $\u^\ast,\v^\ast \in \mathbb{R}_{+}^{n}$, the unique solution $\T_\eps^\ast$ takes the form
\begin{equation}\label{eq:rot-solu}
\T_\eps^\ast = \operatorname{diag}(\u^\ast) \K  \operatorname{diag}(\v^\ast).
\end{equation}
Based on the equation~\eqref{eq:rot-solu}, $\T_\eps^\ast$ can be approximated by the celebrated Sinkhorn algorithm using iterative matrix scaling \citep{sinkhorn1967concerning, cuturi2013sinkhorn}, requiring a computational cost of the order $O(Ln^2)$. The pseudocode for the Sinkhorn algorithm is shown in Algorithm~\ref{alg:sink-ot}, where the operator $\oslash$ denotes the element-wise division.

\begin{algorithm}
\caption{\textsc{SinkhornOT}($\K, \a, \b, \delta$)}
\begin{algorithmic}[1]
\State {\bf Initialize:}
$t \leftarrow 0; \v^{(0)} \leftarrow \mathbf{1}_n$
\Repeat
    \State $t \leftarrow t+1$
    \State $\u^{(t)} \leftarrow \a \oslash \K \v^{(t-1)} ; \quad \v^{(t)} \leftarrow \b \oslash \K^{\top} \u^{(t)}$
\Until{$\|\u^{(t)}-\u^{(t-1)}\|_1 + \|\v^{(t)}-\v^{(t-1)}\|_1 \le \delta$}
\State {\bf Output:} $\T_\eps^\ast = \operatorname{diag}(\u^{(t)}) \K  \operatorname{diag}(\v^{(t)})$
\end{algorithmic}
\label{alg:sink-ot}
\end{algorithm}

\subsection{Unbalanced Optimal Transport Problem}
When $\a,\b \in \mathbb{R}_{+}^{n}$ are two arbitrary positive measures such that their total mass does not equal each other, the marginal constraints in the classical OT problem~\eqref{eq:ot} are no longer valid. 
To overcome such an obstacle, researchers extended the classical optimal transport problem to the so-called unbalanced optimal transport (UOT) problem by relaxing the marginal constraints.
In the literature, there exist several different formulations of the UOT problem; see \cite{liero2016optimal} and \cite{chizat2018unbalanced} for reference.
In this paper, we focus on the static formulation that only involves a minor modification of the initial linear program of OT.
Specifically, the UOT problem between $\a$ and $\b$ is defined as 
\begin{equation}\label{eq:uot}
\UOT(\a, \b):=\min _{\T \in \mathbb{R}_{+}^{n\times n}}\langle \T, \C\rangle + \lambda \mathrm{KL}(\T \mathbf{1}_{n} \| \a) + \lambda \mathrm{KL} (\T^{\top} \mathbf{1}_{n} \| \b).
\end{equation}
Here, $\lam > 0$ is a regularization parameter that balances the trade-off between the transportation effort and maintaining the global structure of input measures. 
Intuitively, mass transportation increases with larger $\lam$.
Note that when $\|\a\|_1 = \|\b\|_1$ and $\lam \rightarrow \infty$, the UOT problem~\eqref{eq:uot} degenerates to the classical OT problem~\eqref{eq:ot}.

One particular solution to the UOT problem is the so-called Wasserstein-Fisher-Rao distance \citep{kondratyev2016new,chizat2018interpolating,liero2018optimal}.
Such a distance is associated with a cost matrix $\C=(C_{ij})$ such that
$C_{ij}=-\log \left[\cos _{+}^{2}\left(d_{ij}/(2\eta)\right)\right]$, where $d_{ij}$ is a distance and $\cos _{+}: z \mapsto \cos \left(\min \left(z, \pi/2\right)\right)$. 
Here, the parameter $\eta$ controls the sparsity level in the kernel matrix $\K$, such that a smaller value of $\eta$ is associated with a sparser $\K$. 
More precisely, when $d_{ij} \ge \pi\eta$, it follows that $C_{ij}=\infty$ and thus $K_{ij}=0$, that is, the transportation between $a_i$ and $b_j$ is blocked.
Hence, a smaller $\eta$ causes more elements in $\K$ to be truncated to zero, resulting in a sparser matrix.
Moreover, a small $\eta$ leads to the diagonal or block-diagonal structure in $\K$, resulting in a large rank of the kernel matrix.
The $\lam$-Wasserstein-Fisher-Rao distance is defined as $\operatorname{WFR}_{\lam}(\cdot, \cdot) := \UOT(\cdot, \cdot)^{1/2}$.
Such a distance has been widely applied in natural language processing \citep{wang2020robust}, earthquake location problems \citep{zhou2018wasserstein}, shape modification, color transfer, and growth models \citep{chizat2018scaling}.

Similar to the classical OT problem, the exact computation of the UOT problem is not scalable in terms of the number of support points $n$. 
Inspired by the success of entropy-regularized OT, we consider the entropic version of the UOT problem, defined as
\begin{equation}\label{eq:ruot}
\UOT_{\lam,\eps}(\a, \b) := \min _{\T \in \mathbb{R}_{+}^{n\times n}}\langle \T, \C\rangle + \lambda \mathrm{KL}(\T \mathbf{1}_{n} \| \a) + \lambda \mathrm{KL} (\T^{\top} \mathbf{1}_{n} \| \b) - \eps H(\T),
\end{equation}
with given parameters $\lam, \eps>0$.
Similarly, we introduce the kernel matrix $\K = \exp(-\C / \eps)$ and solve the problem~\eqref{eq:ruot} by iterative matrix scaling.
Algorithm~\ref{alg:sink-uot} proposed by \cite{chizat2018scaling} is a straightforward generalization of the Sinkhorn algorithm from OT problems to UOT problems. 
The output of Algorithm~\ref{alg:sink-uot} is the unique solution to the problem~\eqref{eq:ruot}. 
Note that when $\lam \rightarrow \infty$, we have $\lam/(\lam+\eps) \to 1$ and thus Algorithm~\ref{alg:sink-uot} degenerates to Algorithm~\ref{alg:sink-ot}.

\begin{algorithm}
\caption{\textsc{SinkhornUOT}($\K, \a, \b, \lam, \eps, \delta$)}
\begin{algorithmic}[1]
\State {\bf Initialize:}
$t \leftarrow 0; \u^{(0)}, \v^{(0)} \leftarrow \mathbf{1}_n$
\Repeat
    \State $t \leftarrow t+1$
    \State $\u^{(t)} \leftarrow \left(\a \oslash \K \v^{(t-1)}\right)^{\lam/(\lam+\eps)}; \quad \v^{(t)} \leftarrow \left(\b \oslash \K^{\top} \u^{(t)}\right)^{\lam/(\lam+\eps)}$
\Until{$\|\u^{(t)}-\u^{(t-1)}\|_1 + \|\v^{(t)}-\v^{(t-1)}\|_1 \le \delta$}
\State {\bf Output:} $\T_{\lam,\eps}^\ast = \operatorname{diag}(\u^{(t)}) \K  \operatorname{diag}(\v^{(t)})$
\end{algorithmic}
\label{alg:sink-uot}
\end{algorithm}

\section{Main Algorithm}\label{sec:method}
In this section, we present our main algorithm called importance sparsification for the Sinkhorn algorithm (\textsc{Spar-Sink}).
The idea is first to apply element-wise subsampling on the kernel matrix $\K$ to obtain a sparse sketch $\tK$.
We then use $\tK$ as a surrogate for $\K$ and use sparse matrix multiplication techniques to accelerate the iterations in the Sinkhorn algorithm.

\subsection{Matrix Sparsification and Importance Sampling}

Given an input matrix $\A$, element-wise matrix sparsification seeks to select (and rescale) a small set of elements from $\A$ and produce a sparse sketch $\tA$, that can serve as a good proxy for $\A$.
Pioneered by \cite{achlioptas2007fast}, previous research has been dedicated to developing various sampling frameworks and probabilities to construct an effective $\tA$ \citep{arora2006fast, candes2010power, drineas2011note, achlioptas2013near, chen2014coherent, gupta2018exploiting}.
Finding such a matrix $\tA$ can not only be used to accelerate matrix operations \citep{drineas2006fast, mahoney2011randomized, gupta2018exploiting, li2023efficient}, but also has broad applications in recovering data with missing features, and preserving privacy when the data cannot be fully observed \citep{kundu2017recovering}.

In this study, we implement the matrix sparsification via the Poisson sampling framework following the recent work of \cite{braverman2021near}. 
Poisson sampling looks at each element and determines whether to include it in the subsample according to a specific probability independently.
Compared to the other commonly used subsampling technique, sampling with replacement, Poisson sampling has a higher approximation accuracy in some situations and is more convenient to implement in distributed systems; see \cite{wang2021comparative} for a comprehensive comparison.

The key to success is how to construct an effective $\tK$ that leads to an asymptotically unbiased solution with a relatively small variance.
To achieve the goal, we develop sampling probabilities based on the idea behind importance sampling, which is widely used for variance-reduction in numerical integration \citep{liu1996metropolized,liu2008monte}.
The importance sampling technique can be described as follows: to approximate the summation $\mu = \sum_{i=1}^N f_i$ with $f_i\ge 0$, we assign each $i \in [N]$ a probability $q_{i} \geq 0$ such that $\sum_{i=1}^{N} q_{i}=1$, and then sample a subset of size $s(<N)$, $\{i_t\}_{t=1}^s$, from $[N]$ based on the probabilities $\{q_{i}\}_{i=1}^N$.
The summation then can be approximated by $\mu\approx\sum_{t=1}^s f_{i_t}/(sq_{i_t})$.
\cite{kahn1953methods} showed when $f_i$'s are known, the optimal sampling probability $q_i$ in terms of variance-reduction is proportional to $f_i$ \citep[Chap.~9]{owen2013monte}.
Despite the effectiveness, such a strategy is not feasible when the values of $f_i$ are unknown or computationally expensive.
Instead, a popular surrogate is using a proper upper bound of $f_i$, denoted by $q_i'$ ($i \in [N]$), as the (un-normalized) sampling probability, such that a higher value of $f_i$ is associated with a larger value of $q_i'$ \citep{owen2013monte, zhao2015stochastic, katharopoulos2018not}.

Following this line of thinking, we reveal a natural upper bound for the elements in the unknown optimal transport plan, and such an upper bound could be used to construct the sampling probability.

\subsection{Importance Sparsification for OT Problems}\label{sec:sparsink-ot}

Recall that our goal is to approximate the entropic OT ``distance''\footnote{Considering its distance-like properties, we employ the term ``distance'' for terminological consistency, despite the fact that the (entropic) OT or UOT distance is not a proper distance.}
\begin{equation}\label{approx-rot}
\OT_\eps(\a, \b) = \langle \T_\eps^\ast,\C \rangle - \eps H(\T_\eps^\ast),  
\end{equation}
where $\C$ is a given cost matrix and $\T_\eps^\ast$ is the unique solution to \eqref{eq:rot}.
To accelerate the Sinkhorn algorithm (i.e., Algorithm~\ref{alg:sink-ot}, illustrated in the left panel of Fig.~\ref{fig:sparsink}), we propose to construct a sparse sketch $\tK$ from $\K$, as shown in the right panel of Fig.~\ref{fig:sparsink}, and compute sparse matrix-vector multiplications, i.e., $\tK \v$ and $\tK^{\top} \u$, in each iteration.
According to the principle of Poisson sampling, the $\tK$ is formulated as follows: given a subsampling parameter $s < n^2$ and a set of sampling probabilities $\{p_{ij}\}_{(i,j)\in [n]\times [n]}$ such that $\sum_{i,j}p_{ij} = 1$, we construct $\tK$ by selecting and rescaling a small fraction of elements from $\K$ and zeroing out the remaining elements, i.e.,
\begin{align}\label{eq:pois}
\widetilde{K}_{i j}= \begin{cases} K_{i j}/p^\ast_{i j} & \text { with prob. } p^\ast_{i j}=\min \left(1, s p_{i j}\right) \\
0 & \text { otherwise. }\end{cases}  
\end{align}
The rescaling factor $p_{ij}^*$ ensures that the sparsified kernel matrix $\tK$ is unbiased w.r.t. $\K$.
Note that 
$\mathbb{E}\{\nnz(\tK)\}=\sum_{i,j} p^\ast_{i j} \leq s \sum_{i,j} p_{i j} = s$,
where $\nnz(\cdot)$ denotes the number of non-zero elements. 
Such an inequality indicates that $s$ is an upper bound of the expected number of non-zero elements in $\tK$.

\begin{figure}[!t]
    \centering
    \includegraphics[width=0.65\linewidth]{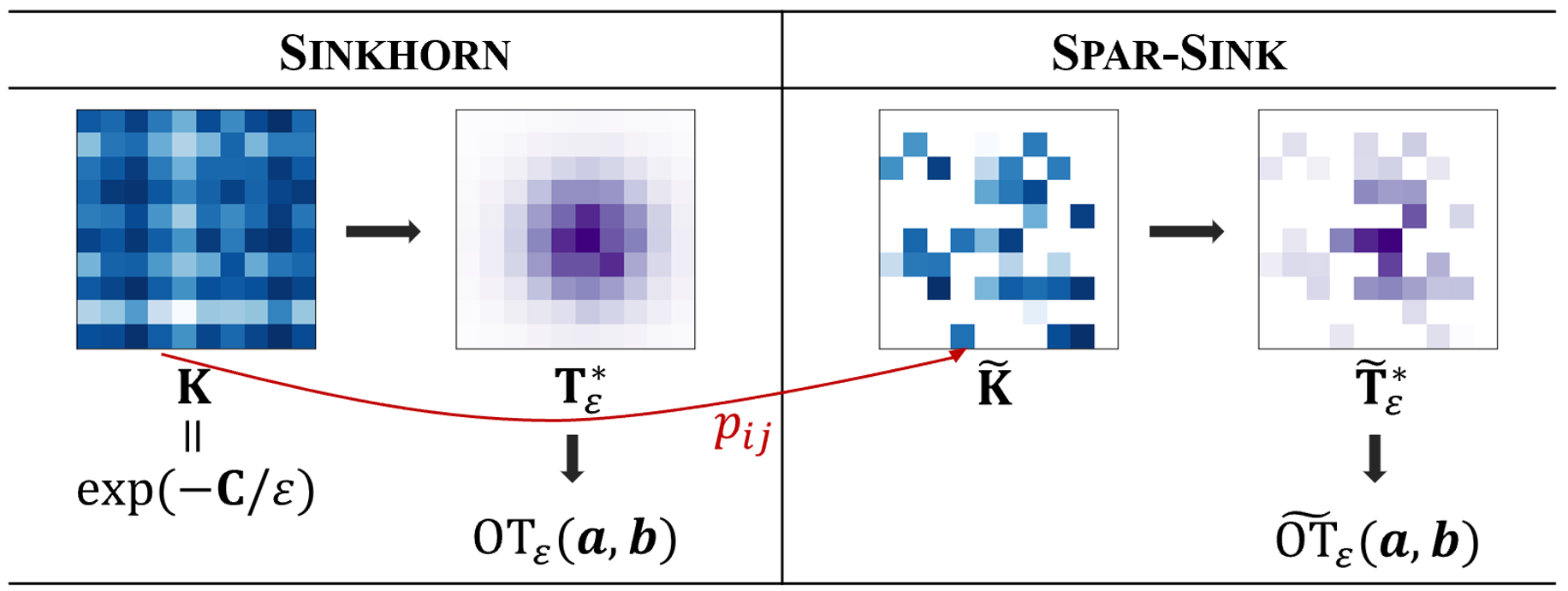}
    \caption{An illustration of the Sinkhorn algorithm (Left panel) and our \textsc{Spar-Sink} method (Right panel). The non-zero elements of each matrix are labeled with colors.}
    \label{fig:sparsink}
\end{figure}

Consider the sampling probabilities $p_{ij}$.
Note that the transportation loss in \eqref{approx-rot} can be written as a summation 
\begin{equation}\label{eq:ot-sum}
\langle \T_\eps^\ast,\C \rangle =\sum_{i,j} {(T_\eps^\ast)}_{ij} C_{ij}.
\end{equation}
According to~\eqref{eq:rot-solu}, $\T_\eps^\ast$ and $\K$ enjoy the same sparsity structure, that is, ${(T_\eps^\ast)}_{ij} = 0$ if $K_{ij} = 0$, as shown in Fig.~\ref{fig:sparsink}.
Thus, sampling elements from $\K$ is equivalent to sampling the corresponding terms from the summation~\eqref{eq:ot-sum}.
Following the idea of importance sampling, the optimal sampling probability $p_{ij}^+$ for $K_{ij}$ should be proportional to ${(T_\eps^\ast)}_{ij} C_{ij}$ from the perspective of variance-reduction.
However, ${(T_\eps^\ast)}_{ij}$ is unknown beforehand, and thus $p_{ij}^+$ is impractical.
Fortunately, there exists a natural upper bound for such a sampling probability.
Based on the marginal constraints on $\T_\eps^\ast$, we have ${(T_\eps^\ast)}_{ij} \le a_i$ and ${(T_\eps^\ast)}_{ij} \le b_j$.
Moreover, we focus on the general scenario where the ground cost between supports is bounded, i.e., $C_{ij} \le c_0$ for some constant $c_0>0$.
Therefore, we have the upper bound 
$${(T_\eps^\ast)}_{ij} C_{ij} \le c_0 \sqrt{a_i b_j}.$$
Such an inequality motivates us to use the sampling probability
\begin{equation}\label{eq:pij-ot}
p_{ij} =\frac{\sqrt{a_i b_j}}{\sum_{1\leq i,j \leq n} \sqrt{a_i b_j}}, \quad 1\leq i,j\leq n.
\end{equation}
Algorithm~\ref{alg:core-ot} summarizes the proposed algorithm for OT problems.

\begin{algorithm}
\caption{\textsc{Spar-Sink} algorithm for OT}
\begin{algorithmic}[1]
\State {\bf Input:}
$\a, \b \in\mathbb{R}_{+}^{n}$, $\K\in\mathbb{R}_{+}^{n \times n}$, $0<s<n^2$, $\eps,\delta > 0$
\State Construct $\tK$ according to~\eqref{eq:pois} and~\eqref{eq:pij-ot}
\State Compute $\tT_\eps^\ast = \textsc{SinkhornOT}(\tK, \a, \b, \delta)$ by using Algorithm~\ref{alg:sink-ot}
\State {\bf Output:} $\widetilde{\OT}_\eps(\a, \b) = \langle \tT_\eps^\ast, \C \rangle - \eps H(\tT_\eps^\ast)$
\end{algorithmic}
\label{alg:core-ot}
\end{algorithm}

\subsection{Importance Sparsification for UOT Problems}\label{sec:uot3.2}

For unbalanced problems, we aim to approximate the entropic UOT distance
\begin{equation}\label{approx-ruot}
\UOT_{\lam,\eps}(\a, \b) = \langle \T_{\lam,\eps}^\ast,\C \rangle + \lam \mathrm{KL}(\T_{\lam,\eps}^\ast \mathbf{1}_{n} \| \a) + \lam \mathrm{KL} (\T_{\lam,\eps}^{\ast\top} \mathbf{1}_{n} \| \b) - \eps H(\T_{\lam,\eps}^\ast),
\end{equation}
where $\T_{\lam,\eps}^\ast$ is the unique solution to~\eqref{eq:ruot}.
Again, we define $\u^\ast, \v^\ast$ as the convergent scaling factors in Algorithm~\ref{alg:sink-uot}, such that $\T_{\lam,\eps}^\ast = \diag(\u^\ast) \K \diag(\v^\ast)$.

Similar to the former subsection, we apply the element-wise Poisson sampling to get an unbiased sparsification of $\K$.
The formulation of $\tK$ is the same as the one in~\eqref{eq:pois}; however, the marginal constraints no longer hold, and thus the sampling probability differs.

Recall that our goal is to find an upper bound of $(T_{\lam,\eps}^\ast)_{ij} C_{ij}$.
According to the iteration steps in Algorithm~\ref{alg:sink-uot}, i.e., 
$(u_i^\ast)^{(\lam+\eps)/\lam} (\sum_{j=1}^{n} K_{ij} v_j^\ast) = a_{i}$ and $(v_j^\ast)^{(\lam+\eps)/\lam} (\sum_{i=1}^{n} K_{ij} u_i^\ast) = b_{j}$, we have
$$
(u_i^\ast)^{\frac{\lam+\eps}{\lam}} K_{ij} v_j^\ast \le a_i, \quad u_i^\ast K_{ij} (v_j^\ast)^{\frac{\lam+\eps}{\lam}} \le b_j \quad \Rightarrow \quad (u_i^\ast)^{\frac{2\lam+\eps}{\lam}} K_{ij}^2 (v_j^\ast)^{\frac{2\lam+\eps}{\lam}} \le a_i b_j
$$
because scaling factors $\u^\ast, \v^\ast$ are non-negative. This follows that
$$ \left(T_{\lam,\eps}^\ast\right)_{ij} = u_i^\ast K_{ij} v_j^\ast \le \left(a_i b_j\right)^{\frac{\lam}{2\lam+\eps}} K_{ij}^{\frac{\eps}{2\lam+\eps}}.$$
Under the scenario that $C_{ij}\leq c_0$, such an upper bound motivates us to sample with the probability
\begin{equation}\label{eq:pij-uot}
p_{ij} =\frac{(a_i b_j)^{\frac{\lam}{2\lam+\eps}} K_{ij}^{\frac{\eps}{2\lam+\eps}}}{\sum_{1\leq i,j \leq n} (a_i b_j)^{\frac{\lam}{2\lam+\eps}} K_{ij}^{\frac{\eps}{2\lam+\eps}}}, \quad 1\leq i,j\leq n.
\end{equation}
Note that when $\lam \to \infty$, the sampling probability $p_{ij}$ defined by~\eqref{eq:pij-uot} degenerates to the one defined in~\eqref{eq:pij-ot}.
This is consistent with the fact that Algorithm~\ref{alg:sink-uot} degenerates to Algorithm~\ref{alg:sink-ot} when $\lam \to \infty$.
Algorithm~\ref{alg:core-uot} summarizes the proposed algorithm for UOT problems.

\begin{algorithm}
\caption{\textsc{Spar-Sink} algorithm for UOT}
\begin{algorithmic}[1]
\State {\bf Input:}
$\a, \b \in\mathbb{R}_{+}^{n}$, $\K\in\mathbb{R}_{+}^{n \times n}$, $0<s<n^2$, $\lam, \eps, \delta > 0$
\State Construct $\tK$ according to~\eqref{eq:pois} and~\eqref{eq:pij-uot}
\State Compute $\tT_{\lam,\eps}^\ast = \textsc{SinkhornUOT}(\tK, \a, \b, \lam, \eps, \delta)$ by using Algorithm~\ref{alg:sink-uot}
\State {\bf Output:} $\widetilde{\UOT}_{\lam,\eps}(\a, \b) = \langle \tT_{\lam,\eps}^\ast, \C \rangle + \lam \mathrm{KL}(\tT_{\lam,\eps}^\ast \mathbf{1}_{n} \| \a) +\lambda \mathrm{KL} (\tT_{\lam,\eps}^{\ast\top} \mathbf{1}_{n} \| \b) - \eps H(\tT_{\lam,\eps}^\ast)$
\end{algorithmic}
\label{alg:core-uot}
\end{algorithm}

In this study, we further extend the \textsc{Spar-Sink} approach to approximate Wasserstein barycenters, by noticing that our importance sparsification mechanism is also applicable for accelerating the iterative Bregman projection algorithm \citep{benamou2015iterative}.
Details for this extension are provided in the Appendix.

\section{Theoretical Results}\label{sec:theory}

This section shows that the proposed estimators w.r.t. entropic OT and UOT distances are consistent under certain regularity conditions.
All the proofs are detailed in the Appendix.
Without loss of generality, we assume the supports of $\a$ and $\b$ are identical\footnote{This is because if $\a$ and $\b$ have two non-overlapping supports, denoted by $\{\x_i\}_{i=1}^n$ and $\{\y_i\}_{i=1}^n$, respectively, one can construct the measures $\tilde{\a},\tilde{\b}\in\Delta^{2n}$, such that $\tilde{\a}=(a_1,\ldots,a_n,0,\ldots,0)$, $\tilde{\b}=(0,\ldots,0,b_1,\ldots,b_n)$.
The measures $\tilde{\a}$ and $\tilde{\b}$ thus share the same support $\{\x_1,\ldots,\x_n,\y_1,\ldots,\y_n\}$.} and $\eps$ is relatively small. Then, the cost matrix $\C$ is symmetric, and the resulting kernel matrix $\K=\exp(-\C/\eps)$ is positive definite.

\begin{theorem}\label{thm:sot}
Suppose the following conditions hold: 
(i) $\|\K\|_2 \ge n^\alpha/c_1$ for some constants $1/2 < \alpha \leq 1$ and $c_1>0$, and the condition number of $\K$ is bounded by $c_2>0$; 
(ii) $p_{ij}^*\ge c_3 s/n^2$ for some constant $c_3>0$; 
(iii) $s\ge c_4 n^{3-2\alpha} \log^4(2n)$, where $c_4=8/(c_3\log^4(1+\epsilon))$ and $\epsilon>0$; and 
(iv) there exist constants $c_a,c_{\alpha}>0$ such that
$n\|\a\|_\infty\le c_a$, $\operatorname{osc}(\boldsymbol{\alpha}^*)\le c_{\alpha}\eps$, $\operatorname{osc}(\bar{\boldsymbol{\alpha}})\le c_{\alpha}\eps$,
where $\operatorname{osc}(\mathbf{x}) := \max_i x_i-\min_i x_i$, and $(\boldsymbol{\alpha}^*,\boldsymbol{\beta}^*)$ and $(\bar{\boldsymbol{\alpha}},\bar{\boldsymbol{\beta}})$ denote the optimal dual potentials for the entropic OT problem and its sparsified counterpart, respectively; see \eqref{eq:dual-ot} and \eqref{eq:dual-ot-spar}.
Then, as $n\to\infty$, the following result holds with probability approaching one:
\begin{equation}\label{eq:ot-thm-ub}
\widetilde{\OT}_\eps(\a, \b) \le \OT_\eps(\a, \b) +  c_5 \eps \sqrt{n^{3-2\alpha}/s},
\end{equation}
where $c_5>0$ is a constant depending on $c_1$, $c_2$, $c_3$, $c_4$, $c_a$, and $c_{\alpha}$ only.
\end{theorem}

We discuss the regularity conditions in Theorem~\ref{thm:sot}.
Condition~(i) is naturally established when $\eps$ is relatively small.
Indeed, non-diagonal entries of $\K$ go to zero quickly as the cost or distance increases, thus yielding a numerically sparse kernel matrix with a diagonal-like structure.
Condition~(ii) requires $p_{ij}$ to be of the order $O(1/n^2)$, which can always be satisfied by combining the proposed sampling probability and uniform sampling probability linearly.
Such a shrinkage strategy is common in subsampling literature, and we refer to \cite{ma2015statistical} and \cite{yu2022optimal} for more discussion.
Condition~(iii) implies that the subsample size should be large enough, especially when the signal in $\K$ is weak (i.e., $\alpha$ is small). Under the condition~(iii), the upper bound of approximation error in \eqref{eq:ot-thm-ub} tends to zero in probability, which leads to the consistency of $\widetilde{\OT}_\eps(\a, \b)$ w.r.t. the entropic OT distance. 
Condition~(iv) imposes a mild marginal regularity together with bounded oscillations of the optimal dual potential.
Moreover, consider a general case that $\|\K\|_2 = O(n)$, i.e., $\alpha = 1$, condition~(iii) indicates us to select $s=\tO(n)$ elements to construct the sparse sketch.

To analyze the UOT problem, we first rescale the source measure such that $\bm a \in \Delta^{n-1}$. Then, the following result holds.

\begin{theorem}\label{thm:suot}
    Suppose the regularity conditions (i)---(iii) in Theorem~\ref{thm:sot} hold. Also suppose that: (iv) there exist constants $c_a,c_{u,\alpha}>0$ such that
    $n\|\a\|_\infty\le c_a$, $\|\boldsymbol{\alpha}^*\|_\infty\le c_{u,\alpha}\eps$, $\|\bar{\boldsymbol{\alpha}}\|_\infty\le c_{u,\alpha}\eps$,
    where $(\boldsymbol{\alpha}^*,\boldsymbol{\beta}^*)$ and $(\bar{\boldsymbol{\alpha}},\bar{\boldsymbol{\beta}})$ denote the optimal dual potentials for the entropic UOT problem and its sparsified counterpart, respectively; see \eqref{eq:dual-uot} and \eqref{eq:dual-uot-spar}; (v) $\varepsilon/\lambda \leq c_6$ for some constant $c_6>0$. As $n\to\infty$, the following result holds with probability approaching one that
	\begin{equation}\label{eq:uot-thm-ub}
	\widetilde{\UOT}_{\lam,\eps} (\a, \b) \le \UOT_{\lam,\eps} (\a, \b)
	+c_8 \varepsilon \sqrt{n^{3-2\alpha}/s},
	\end{equation}
	where $c_8>0$ is a constant depending on $c_1$, $c_2$, $c_3$, $c_4$, $c_a$, $c_{u,\alpha}$, and $c_6$ only.
\end{theorem}

Consider the additional conditions in Theorem~\ref{thm:suot}.
Condition~(iv) imposes a mild marginal regularity together with bounded UOT dual potentials.
Condition~(v) naturally holds when $\eps=O(\lambda)$.
Theorem~\ref{thm:suot} shows the consistency of $\widetilde{\UOT}_{\lam,\eps} (\a, \b)$ w.r.t. the entropic UOT distance, and also requires $s$ to be at least of the order $\tO(n)$.

The following theorem shows that the proposed \textsc{Spar-Sink} algorithm has the same number of iteration bound as the classical Sinkhorn algorithm up to a constant, for both OT and UOT problems. This result is a straightforward extension of the iteration bounds presented in \cite{altschuler2017near} and \cite{pham2020unbalanced}.

\begin{theorem}\label{thm:time}
Suppose the Sinkhorn algorithm and \textsc{Spar-Sink} algorithm have the same settings of parameters. 
Under the conditions of Theorem~\ref{thm:sot} (resp. Theorem~\ref{thm:suot}), both Algorithm~\ref{alg:sink-ot} and Algorithm~\ref{alg:core-ot} (resp. Algorithm~\ref{alg:sink-uot} and Algorithm~\ref{alg:core-uot}) converge approximately within the same order of iterations in probability.
\end{theorem}

\section{Simulations}\label{sec:simu}

In this section, we evaluate the performance of our proposed method (\textsc{Spar-Sink}) in both OT and UOT problems using synthetic data sets.
We compare \textsc{Spar-Sink} with state-of-the-art variants of Sinkhorn regarding approximation accuracy and computational time, including: (i) \textsc{Greenkhorn} \citep{altschuler2017near}; (ii) \textsc{Screenkhorn} \citep{alaya2019screening}; (iii) \textsc{Nys-Sink} \citep{altschuler2019massively}; (iv) the naive random element-wise subsampling method in the Sinkhorn algorithm (\textsc{Rand-Sink}), which is similar to the proposed \textsc{Spar-Sink} method, except that the sampling probabilities for all the elements are equal to each other.

We set the stopping threshold $\delta = 10^{-6}$ for all the algorithms considered in the experiments. The maximum number of iterations is set to be $5n$ for \textsc{Greenkhorn} and to be $10^3$ for all other methods. The decimation factor in \textsc{Screenkhorn} is taken as $3$. Other parameters are set by default according to the Python Optimal Transport toolbox \citep{flamary2021pot}.
All experiments are implemented on a server with 251GB RAM, 64 cores Intel(R) Xeon(R) Gold 5218 CPU and 4 GeForce RTX 3090 GPU.
The implementation code is available at this link: \url{https://github.com/Mengyu8042/Spar-Sink}.

\subsection{Approximation Performance}\label{sec:simu-subsec1}

For the OT problem, the goal is to estimate the entropic OT distance between two empirical probability measures $\a, \b \in \Delta^{n-1}$, i.e., $\OT_\eps(\a,\b)$ defined in~\eqref{eq:rot}.
These two measures share the same support points $\{\x_i\}_{i=1}^n$, where $\x_i\in\RR^d$, $n=10^3$ and $d \in \{5, 10, 20, 50\}$.
We use the squared Euclidean cost matrix $\C$ such that $C_{ij} = \|\x_i - \x_j\|_2^2$ for $1\leq i,j\leq n$, and we take $\eps \in \{10^{-1},10^{-2},10^{-3}\}$.
In addition, we consider three scenarios for generating $\a, \b$ and $\{\x_i\}_{i=1}^n$ as follows:

\begin{itemize}
\item[\textbf{C1.}] $\a, \b$ are empirical Gaussian distributions $N(\frac{1}{3},\frac{1}{20})$ and $N(\frac{1}{2},\frac{1}{20})$, respectively; $\x_i$'s are generated from multivariate uniform distribution over $(0,1)^d$, i.e., $\x_i \sim U(0,1)^d$; 
\item[\textbf{C2.}] $\a, \b$ are same to those in \textbf{C1}; $\x_i$'s are generated from multivariate Gaussian distribution, i.e., $\x_i \sim N(\mathbf{0}_{d}, \mathbf{\Sigma})$ with $\Sigma_{jk}=0.5^{|j-k|}$ for $(j,k)\in [d]\times [d]$;
\item[\textbf{C3.}] $\a, \b$ are empirical t-distributions with 5 degrees of freedom $t_5(\frac{1}{3},\frac{1}{20})$ and $t_5(\frac{1}{2},\frac{1}{20})$, respectively; $\x_i$'s are same to those in \textbf{C1}.
\end{itemize}

We first compare the subsampling-based approaches: \textsc{Nys-Sink}, \textsc{Rand-Sink}, and \textsc{Spar-Sink} (i.e., Algorithm~\ref{alg:core-ot}).
For the \textsc{Rand-Sink} and \textsc{Spar-Sink} methods, we set the expected subsample size $s = \{2,2^{2},2^{3},2^{4}\}\times s_0(n)$ with $s_0(n) = 10^{-3} n \log^4(n)$, where $s_0(n)$ is set in the light of Theorem~\ref{thm:sot}.
For a fair comparison, we select $r = \lceil s/n \rceil$ columns in $\K$ for the \textsc{Nys-Sink} approach, such that the selected elements for the subsampling-based methods are roughly at the same size.
To compare the approximation performance, we calculate the empirical relative mean absolute error (RMAE) for each estimator based on 100 replications, i.e., 
$$
\operatorname{RMAE}^{(\OT)} = \frac{1}{100}\sum_{i=1}^{100} \frac{|\widetilde{\OT}_{\eps}^{(i)}-{\OT}_{\eps}^{(i)}|}{{\OT}_{\eps}^{(i)}},
$$
where $\widetilde{\OT}_{\eps}^{(i)}$ represents the estimator in the $i$th replication, and ${\OT}_{\eps}^{(i)}$ is calculated using the classical Sinkhorn algorithm (i.e., Algorithm~\ref{alg:sink-ot}).

The results of $\operatorname{RMAE}^{(\OT)}$ versus different subsample sizes $s$ are shown in Fig.~\ref{fig:simu-ot}. From Fig.~\ref{fig:simu-ot}, we observe that all the estimators result in smaller $\operatorname{RMAE}^{(\OT)}$ as $s$ increases, and the proposed \textsc{Spar-Sink} method consistently outperforms the competitors. 
We also observe that \textsc{Spar-Sink} decreases faster than others in most cases, which indicates the proposed method has a relatively high convergence rate.

\begin{figure}[!t]
    \centering
    \includegraphics[width=0.8\linewidth]{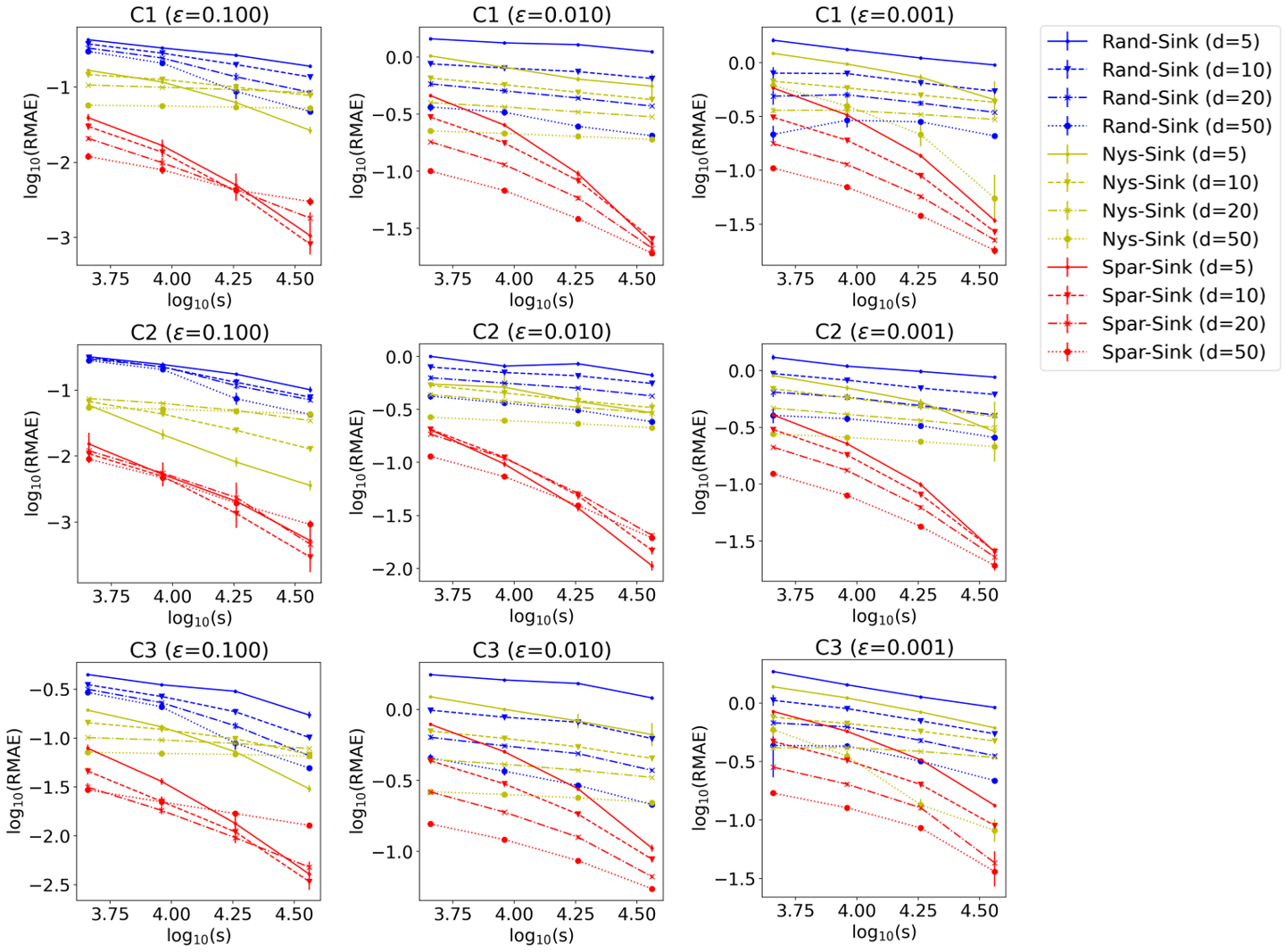}
    \caption{Comparison of subsampling-based methods w.r.t. $\operatorname{RMAE}^{(\OT)}$ versus increasing $s$ (in log-log scale). Each row represents a different data generation pattern (\textbf{C1---C3}), and each column represents a different $\eps$. Different methods are marked by different colors, respectively, and each line type represents a different dimension $d$. Vertical bars are the standard errors.}
    \label{fig:simu-ot}
\end{figure}

For the UOT problem shown in~\eqref{eq:ruot}, we set the total mass of $\a$ and $\b$ to be 5 and 3, respectively.
The regularization parameters are set to be $\eps=0.1$ and $\lam=0.1$. 
Other choices of parameters lead to similar results and are relegated to Appendix.
Empirical results show the performance of the proposed method is robust to these parameters.
The goal is to approximate the Wasserstein-Fisher-Rao distance, where the cost function is defined as $C_{ij} = -\log \left\{\cos _{+}^{2}\left(d_{ij}/(2\eta)\right)\right\}$ with Euclidean distance $d_{ij} = \|\x_i - \x_j\|_2$.
Recall that the parameter $\eta$ controls the sparsity level of the kernel matrix $\K$, and a smaller $\eta$ is associated with a sparser $\K$.
We take different values of $\eta$ such that there are around 70\%, 50\%, and 30\% non-zero elements in $\K$, and these scenarios are denoted by \textbf{R1}, \textbf{R2}, and \textbf{R3}, respectively. 
Other settings are the same as those in OT problems.

For comparison, we calculate the empirical RMAE of approximating ${\UOT}_{\lam,\eps}(\a, \b)$ based on 100 replications, i.e., $$
\operatorname{RMAE}^{(\UOT)} = \frac{1}{100} \sum_{i=1}^{100} \frac{|\widetilde{\UOT}_{\lam,\eps}^{(i)}-{\UOT}_{\lam,\eps}^{(i)}|}{{\UOT}_{\lam,\eps}^{(i)}},
$$
where $\widetilde{\UOT}_{\lam,\eps}^{(i)}$ represents the estimator in the $i$th replication, and ${\UOT}_{\lam,\eps}^{(i)}$ is calculated using the unbalanced Sinkhorn algorithm (i.e., Algorithm~\ref{alg:sink-uot}).
The results of $\operatorname{RMAE}^{(\UOT)}$ versus different $s$ are shown in Fig.~\ref{fig:simu-uot}, from which we observe that $\operatorname{RMAE}^{(\UOT)}$ of both \textsc{Rand-Sink} and \textsc{Nys-Sink} methods decrease slowly with the increase of $s$, while \textsc{Spar-Sink} converges much faster.
In general, the proposed \textsc{Spar-Sink} significantly outperforms the competitors under all circumstances. 
Such an observation indicates the proposed algorithm can select informative elements for the Sinkhorn algorithm, resulting in an asymptotically unbiased result with a relatively small estimation variance.

\begin{figure}[!t]
    \centering
    \includegraphics[width=0.8\linewidth]{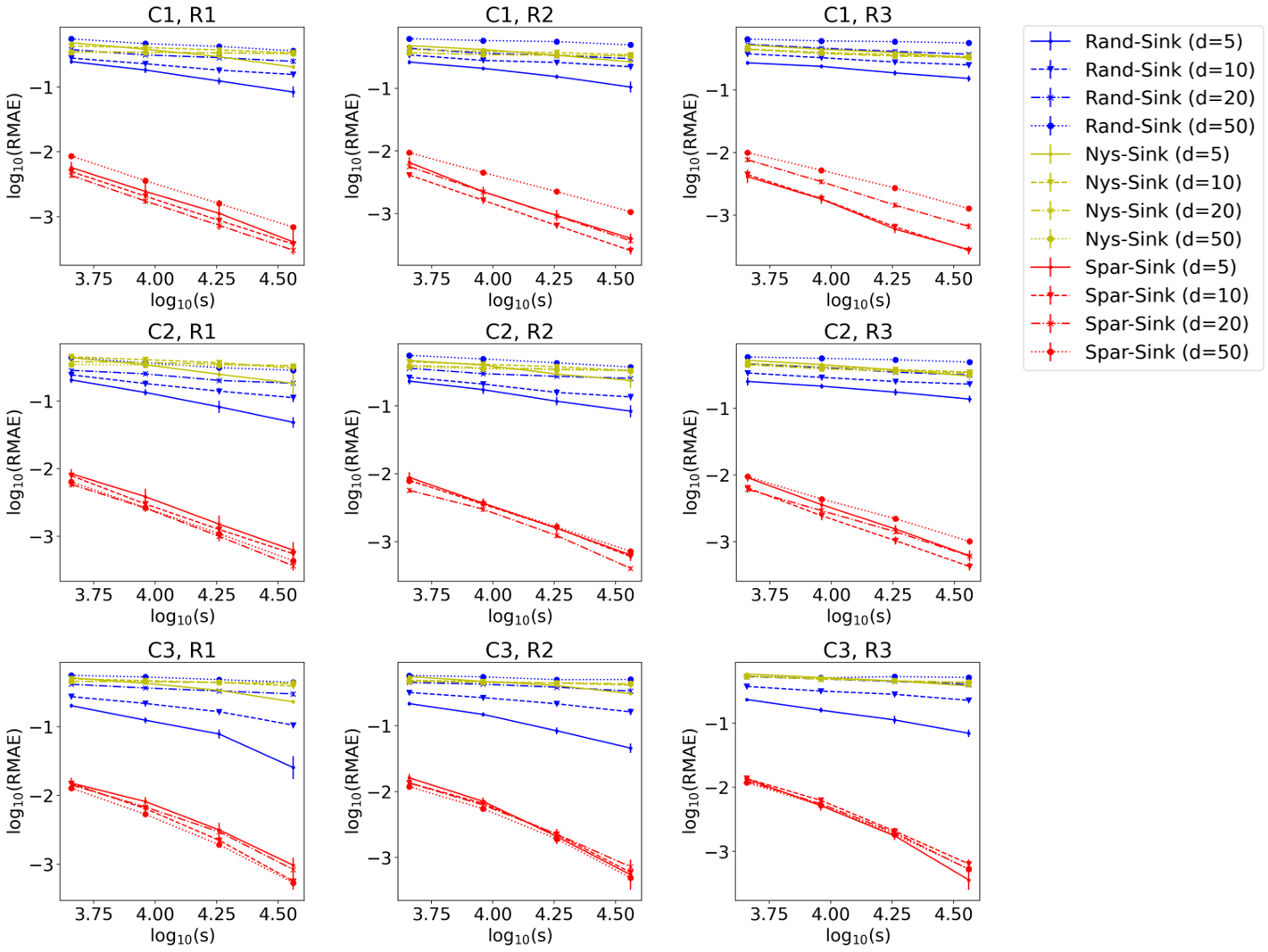}
    \caption{Comparison of subsampling-based methods w.r.t. $\operatorname{RMAE}^{(\UOT)}$ versus increasing $s$ (in log-log scale). Each row represents a different data generation pattern (\textbf{C1---C3}), and each column represents a different sparsity ratio (\textbf{R1---R3}). Different methods are marked by different colors, respectively, and each line type represents a different dimension $d$. Vertical bars are the standard errors.}
    \label{fig:simu-uot}
\end{figure}

We now include the methods without subsampling, \textsc{Greenkhorn} and \textsc{Screenkhorn}, to comparison and fix the subsample parameter as $s=8s_0(n)$ for above subsampling-based approaches. We show their $\operatorname{RMAE}^{(\OT)}$ versus increasing sample size $n$ under \textbf{C1} in Fig.~\ref{fig:simu-greenkhorn}, where $n\in\{2^2, 2^3, \ldots, 2^7\}\times 10^2$. 
We omit the result of \textsc{Screenkhorn} in the case of $\eps=10^{-3}$ as it fails to output a feasible solution when $\eps$ is relatively small in our setup. 
From Fig.~\ref{fig:simu-greenkhorn}, we observe the proposed \textsc{Spar-Sink} method yields comparable errors to \textsc{Greenkhorn} and \textsc{Screenkhorn} for a relatively large $\eps$, and its advantage turns prominent when $\eps$ becomes small. Additionally, the approximation error of \textsc{Spar-Sink} converges asymptotically as $n$ increases, which is consistent with Theorem~\ref{thm:sot}. We also show the convergence of $\operatorname{RMAE}^{(\UOT)}$ versus increasing $n$ in the Appendix, and the results are in good agreement with Theorem~\ref{thm:suot}.

\begin{figure}[!t]
    \centering
    \includegraphics[width=0.8\linewidth]{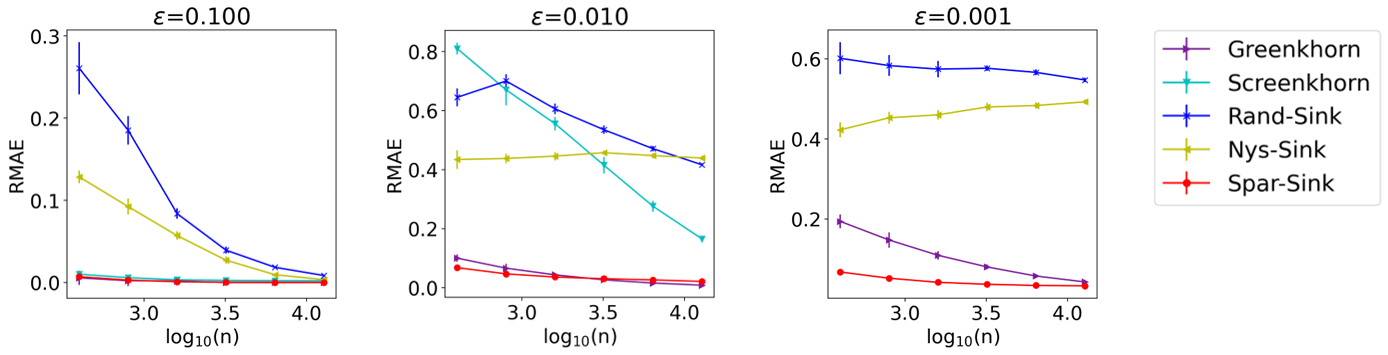}
    \caption{Comparison of different methods w.r.t. $\operatorname{RMAE}^{(\OT)}$ versus $\log_{10}(n)$ under \textbf{C1}. Each subfigure represents a different $\eps$, and each color marks a specific method. Vertical bars are the standard errors.}
    \label{fig:simu-greenkhorn}
\end{figure}

\subsection{Computational Cost and CPU Time}

Consider the computational cost of Algorithm~\ref{alg:core-ot}.
Constructing the sketch $\tK$ requires $O(n^2)$ time, and such a step can be naturally paralleled.
The matrix $\tK$ contains at most $s$ non-zero elements, and thus calculating $\tK \v$ and $\tK^{\top} \u$ takes $O(s)$ time.
Therefore, the overall computational cost of Algorithm~\ref{alg:core-ot} is at the order of $O(n^2 + Ls)$, which becomes $O(n^2 + Ln)$ when $s=\widetilde{O}(n)$.
Similarly, the computational cost of Algorithm~\ref{alg:core-uot} is at the order of $O(\nnz(\K) +Ln)$, where $\nnz(\cdot)$ denotes the number of non-zero elements.

We compare the CPU time of the classical Sinkhorn algorithm and the variants of Sinkhorn for both OT and UOT problems in Fig.~\ref{fig:simu-time}.
The \textsc{Rand-Sink} method has similar computing time to \textsc{Spar-Sink} and is omitted for clarity.
We choose $s = 8 s_0(n)$ for \textsc{Spar-Sink} and $r=\lceil s/n \rceil$ for \textsc{Nys-Sink} with $n\in\{2^3, 2^4, \ldots, 2^8\}\times 10^2$.

\begin{figure}[!t]
    \centering
    \subfigure[CPU time (in seconds) for estimating the Wasserstein distance under \textbf{C1}.]{
    \includegraphics[height=3.3cm]{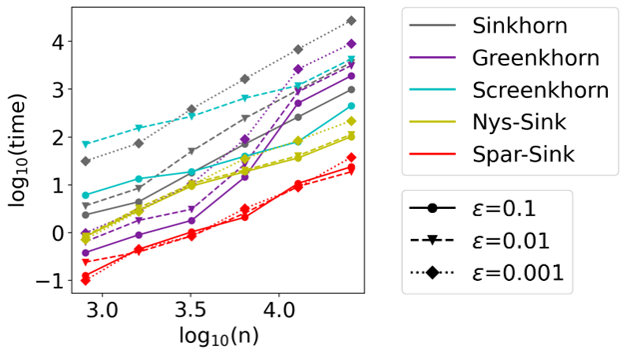}\label{fig:simu-ot-time}
    }
    \hspace{1cm}
    \subfigure[CPU time (in seconds) for estimating the WFR distance under \textbf{C1, R2}.]{
    \includegraphics[height=3.3cm]{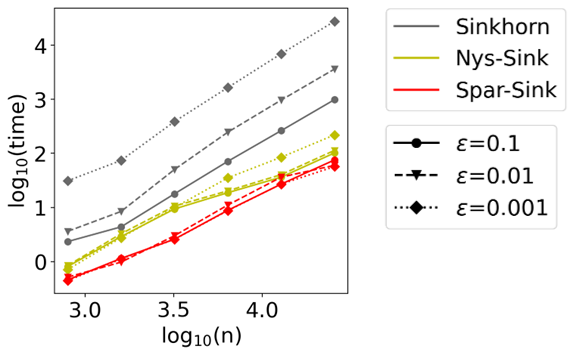}\label{fig:simu-uot-time}
    }
    \caption{Comparison of different methods w.r.t. computational time. Different methods are marked by different colors. Each line type represents a different value of $\eps$.}
    \label{fig:simu-time}
\end{figure}

In Fig.~\ref{fig:simu-time}, we observe that \textsc{Spar-Sink} speeds up the Sinkhorn algorithm hundreds of times and also computes much faster than \textsc{Greenkhorn} and \textsc{Screenkhorn}, especially when $n$ is large enough.
We also observe that a smaller value of $\eps$ leads to a longer CPU time for the Sinkhorn algorithm.
Such an observation is consistent with the results in \cite{altschuler2017near} and \cite{pham2020unbalanced}, which showed the number of iterations for Sinkhorn increases as $\eps$ decreases.
In contrast, the effect of $\eps$ on the CPU time is less significant for \textsc{Spar-Sink}. 
These observations indicate that the proposed algorithms are suitable for dealing with large-scale OT and UOT problems.

\section{Echocardiogram Analysis}\label{sec:real}

Echocardiography has been widely used to visualize myocardial motion due to its fast image acquisition, relatively low cost, and no side effects. 
Previous study has developed various echocardiology-based techniques to determine the ejection fraction \citep{ouyang2020video}, prognosticate cardiovascular disease \citep{zhang2021ensemble}, screen the cardiotoxicity \citep{bouhlel2020early}, among others.
One fundamental task in echocardiogram data analysis is cardiac circle identification, which is necessary and crucial for downstream analysis. 
The cardiac cycle is the performance of the human heart from the beginning of one heartbeat to the beginning of the next. 
A single cycle consists of two basic periods, diastole and systole \citep{fye2015caring}.
Owing to the variation in cardiac activity caused by changes in loading and cardiac conditions, it is recommended to consider multiple cycles rather than only one representative cycle to perform measurements.
However, this is not always done in clinical practice, given the tedious and laborious nature of human labeling.
To obviate the heavy work for cardiologists, we propose an optimal transport method to automatically identify and visualize multiple cardiac cycles.

We consider an echocardiogram videos data set \citep{ouyang2020video} containing 10,030 apical-four-chamber echocardiogram videos, each of which ranges from 24 to 1,002 frames with an average of 51 frames per second. 
A single frame is a gray-scale image of $112 \times 112$ pixels.
Each video is annotated with two separate time points representing the end-systole (ES) and the end-diastole (ED). 
Figure~\ref{fig:echo-video} gives a clip example of the echocardiogram videos data set.

\begin{figure}[!t]
    \centering
    \includegraphics[width=0.95\linewidth]{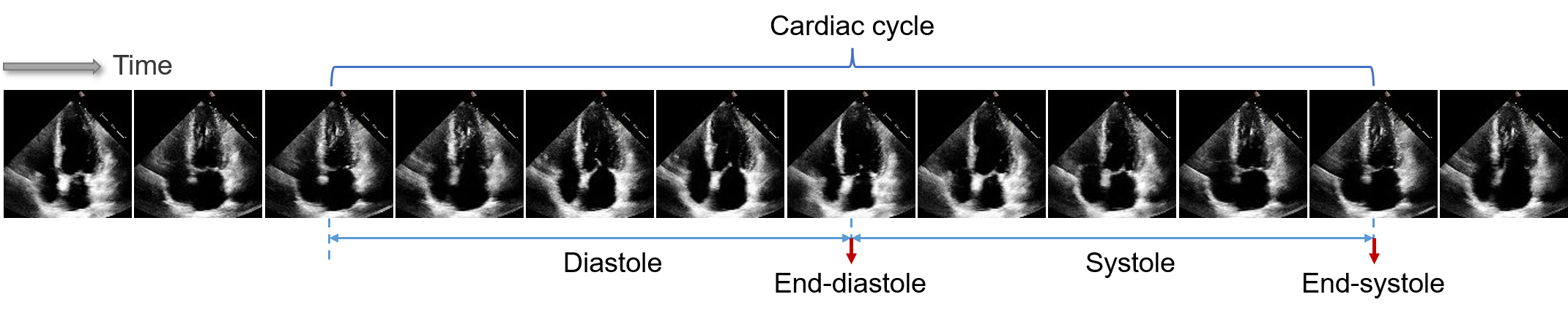}
    \caption{Echocardiogram videos data set visualization. Two basic periods, diastole and systole, form a cardiac cycle.}
    \label{fig:echo-video}
\end{figure}

We propose identifying cardiac cycles using pairwise distances between the frames in an echocardiogram video.
In particular, we use the normalized pixel gray levels of each frame as a mass distribution supported on $\RR^2$, such that a lighter color is associated with a larger mass.
We then use the Wasserstein-Fisher-Rao distance to measure the dissimilarity between each pair of frames.
Compared to the Wasserstein distance, the WFR distance prevents long-range mass transportation and thus can achieve a balance between global transportation and local truncation. 
Intuitively, such a distance is more consistent with the characteristics of myocardial motion that the cardiac muscle would not move largely.
To identify the cardiac cycles of an individual, we first compute the pairwise WFR distance matrix of his/her video, and then conduct a multidimensional scaling (MDS) for the distance matrix. 
However, computing the full WFR distance matrix using the classical Sinkhorn algorithm for a video of 200 frames requires nearly a hundred days.
To alleviate the computational burden, we sample every other two frames (sampling period of 3) and then use our proposed \textsc{Spar-Sink} algorithm to approximate the pairwise WFR distances of the downsampled videos.
The parameters are set to be $\eps = 0.01$, $\lambda=1$, $\eta = 15$, and $s=8s_0(n)$. Empirical results show the performance is not sensitive to these parameters.
Our CPU implementation requires only a few hours to calculate the distance matrix for one video.
Further acceleration using GPU implementation is left for future research.

Figure~\ref{fig:echo-cycle} visualizes the distance matrices and the MDS results w.r.t. three individuals, respectively.
Each dot in the MDS result represents a single frame, and the time points w.r.t. frames are denoted by different colors. 
By connecting the dots sequentially according to the time points, the cyclical nature of cardiac activities is clearly recovered.
Moreover, we can make a preliminary assessment of one's cardiac function from the pattern of these cardiac circles.
For instance, by comparing with the first individual from the control group, we can see that the circle size differs in different cycles for the third individual with arrhythmia.

\begin{figure}[!t]
    \centering
    \includegraphics[width=0.8\linewidth]{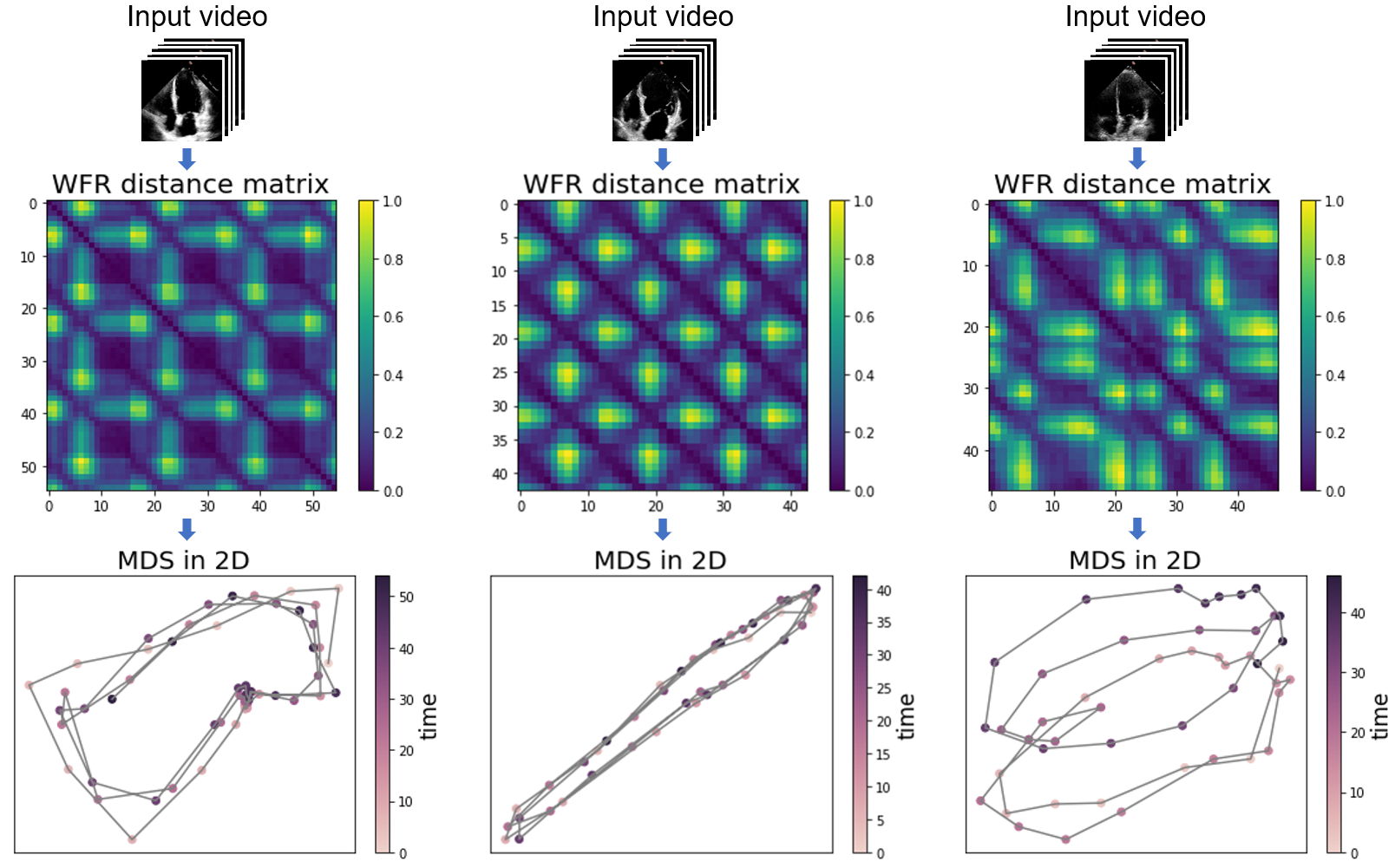}
    \caption{(From left to right) Each column is associated with an individual corresponding to a specific state of cardiac function, i.e., health, heart failure, and arrhythmia, respectively. (Top row) Input echocardiogram videos. (Middle row) Normalized WFR distance matrices computed by the \textsc{Spar-Sink} algorithm. (Bottom row) MDS in 2D: each point corresponds to a frame and is colored by the corresponding time iteration.}
    \label{fig:echo-cycle}
\end{figure}

Beyond the intuitive visualization aforementioned, we are also interested in the accuracy of cycle prediction. 
Therefore, we consider the task of ED time point prediction.
Specifically, for each video, we use the manually annotated ES and ED time points, $t_{ES}$ and $t_{ED}$, as the ground truth, and we aim to predict $t_{ED}$ using $t_{ES}$. 
Intuitively, in one cardiac cycle, the ED frame is the one that is most dissimilar to the ES frame.
Following this line of thinking, we calculate the WFR distances between the ES frame and the other frames, respectively, within one cardiac cycle, and the predicted ED frame is the one that yields the largest WFR distance.
After obtaining the prediction $\hat{t}_{ED}$, we define its error as 
$$\text{Error} = \left|1-\frac{\hat{t}_{E D}-t_{E S}}{t_{E D}-t_{E S}}\right|.$$
A smaller error implies the prediction is closer to the ground truth.

\begin{table}[t]
\centering
\label{tab:echo}
\resizebox{\textwidth}{!}{%
\begin{tabular}{ccccccc}
\multicolumn{7}{c}{(a) Original scale ($n=112\times 112$).}  \\ 
\hline
\multicolumn{1}{c}{}                &          & $s=s_0(n)$ & $s=2s_0(n)$ & $s=2^{2}s_0(n)$ & $s=2^{3}s_0(n)$  & $n^2$         \\
\hline
\multirow{2}{*}{\textsc{Nys-Sink}}  & Error    & 0.49$_{\pm 0.23}$         & 0.44$_{\pm 0.27}$         & 0.31$_{\pm 0.28}$         & 0.32$_{\pm 0.22}$          &   -            \\
                                    & Time     & 370.30             & 522.37             & 662.71             & 831.01              &   -            \\
\multirow{2}{*}{\textsc{Robust-NysSink}}  & Error    & 0.47$_{\pm 0.34}$         & 0.41$_{\pm 0.25}$         & 0.34$_{\pm 0.18}$         & 0.33$_{\pm 0.17}$          &   -            \\
                                    & Time     & 376.01             & 509.58             & 664.50            & 838.61              &   -            \\
\multirow{2}{*}{\textsc{Rand-Sink}} & Error    & 0.21$_{\pm 0.14}$         & 0.13$_{\pm 0.08}$         & 0.11$_{\pm 0.08}$         & 0.09$_{\pm 0.06}$          &   -            \\
                                    & Time     & 181.26             & 226.32             & 251.68             & 314.56              &   -            \\
\multirow{2}{*}{\textsc{Spar-Sink}} & Error    & \textbf{0.09}$_{\pm 0.06}$         & \textbf{0.07}$_{\pm 0.05}$         & \textbf{0.06}$_{\pm 0.05}$         & \textbf{0.06}$_{\pm 0.04}$          &   -            \\
                                    & Time     & 210.69             & 262.89             & 302.9              & 357.46              &   -            \\
\multirow{2}{*}{Sinkhorn}  & Error    &    -               &   -                &   -                &   -                 & 0.06$_{\pm 0.05}$  \\
                                    & Time     &    -               &   -                &   -                &   -                 & 15649.31      \\
\hline
\multicolumn{7}{c}{(b) Mean-pooling with $2 \times 2$ filters and stride 2 ($n=56\times 56$).}  \\ 
\hline
\multicolumn{1}{c}{}                &          & $s=s_0(n)$ & $s=2s_0(n)$ & $s=2^{2}s_0(n)$ & $s=2^{3}s_0(n)$  & $n^2$         \\
\hline
\multirow{2}{*}{\textsc{Nys-Sink}}  & Error    & 0.78$_{\pm 0.21}$         & 0.64$_{\pm 0.26}$         & 0.55$_{\pm 0.27}$         & 0.45$_{\pm 0.26}$          &   -           \\
                                    & Time     & 40.51              & 46.97              & 51.46              & 59.54               &   -           \\
\multirow{2}{*}{\textsc{Robust-NysSink}}  & Error    & 0.79$_{\pm 0.31}$         & 0.61$_{\pm 0.27}$         & 0.50$_{\pm 0.24}$         & 0.43$_{\pm 0.22}$          &   -           \\
                                    & Time     & 40.75              & 46.11              & 50.08              & 56.20               &   -           \\
\multirow{2}{*}{\textsc{Rand-Sink}} & Error    & 0.38$_{\pm 0.29}$         & 0.35$_{\pm 0.32}$         & 0.28$_{\pm 0.27}$         & 0.16$_{\pm 0.13}$          &   -           \\
                                    & Time     & 22.56              & 23.73              & 25.38              & 27.45               &   -           \\
\multirow{2}{*}{\textsc{Spar-Sink}} & Error    & \textbf{0.30}$_{\pm 0.21}$         & \textbf{0.14}$_{\pm 0.11}$         & \textbf{0.11}$_{\pm 0.09}$         & \textbf{0.11}$_{\pm 0.09}$          &   -           \\
                                    & Time     & 24.34              & 27.24              & 28.59              & 32.43               &   -           \\
\multirow{2}{*}{Sinkhorn}  & Error    &    -               &   -                &   -                &   -                 & 0.11$_{\pm 0.10}$  \\
                                    & Time     &    -               &   -                &   -                &   -                 & 668.81        \\
\hline                                    
\end{tabular}
}
\caption{Average errors (with standard deviations presented in footnotes) and CPU time (in seconds) for predicting the ED time point. }
\end{table}

We calculate the WFR distances using the Sinkhorn algorithm, as well as three subsampling algorithms, i.e., \textsc{Rand-Sink}, \textsc{Nys-Sink}, and the proposed \textsc{Spar-Sink} algorithm, under different subsample sizes. 
Considering the potential outliers in real data, we also include a robust variant of \textsc{Nys-Sink} (\textsc{Robust-NysSink}) proposed by \cite{le2021robust} for comparison.
The results for 100 randomly selected videos are reported in Table~\ref{tab:echo}. 
From panel~(a) in Table~\ref{tab:echo}, we observe that all these subsampling algorithms yield significantly less CPU time than the classical Sinkhorn algorithm.
In addition, \textsc{Spar-Sink} is as accurate as Sinkhorn, while \textsc{(Robust-)Nys-Sink} and \textsc{Rand-Sink} yield much larger errors. 

Another interesting question is how the proposed algorithm compares with pooling techniques, which are widely used in computer vision to accelerate computation \citep{boureau2010learning, gong2014multi, xu2022regularized}.
To answer this question, we reduce the size of the images from $112\times 112$ to $56\times 56$ using mean-pooling with $2\times 2$ filters and stride 2.
We then compute the WFR distances on these pooled images.
The results are provided in panel~(b) of Table~\ref{tab:echo}, from which we observe that all the algorithms require significantly less CPU time for the pooled images; however, the error increases.
Again, the proposed \textsc{Spar-Sink} algorithm is the only one that yields the same error as the Sinkhorn algorithm.
We also observe that compared to the Sinkhorn algorithm for pooled images (i.e., Sinkhorn in panel~b), the \textsc{Spar-Sink} for original images (i.e., \textsc{Spar-Sink} in panel~a) requires a shorter CPU time and yields more minor errors.
Such an observation indicates that the proposed algorithm could be a better alternative for pooling strategy when calculating transport distances between large-scale images. 
In addition, one can also combine the pooling strategy with the proposed algorithm to further reduce CPU time without loss of accuracy.

Besides the application in echocardiogram analysis, we evaluate the \textsc{Spar-Sink} approach in two common machine learning applications: color transfer and generative modeling. The experimental results are presented in the Appendix, which provides additional evidence of the effectiveness of our proposed method.

\section{Concluding Remarks}\label{sec:conclude}

Realizing the natural upper bounds for unknown transport plans in (unbalanced) optimal transport problems, we propose a novel importance sparsification method to accelerate the Sinkhorn algorithm, approximating entropic OT and UOT distances in a unified framework.
Theoretically, we show the consistency of proposed estimators under mild regularity conditions.
Experiments on various synthetic data sets demonstrate the accuracy and efficiency of the proposed \textsc{Spar-Sink} approach. 
We also consider an echocardiogram video data set to illustrate its application in cardiac cycle identification, which shows our method offers a great trade-off between speed and accuracy.

Inspired by the work of~\cite{xie2020fast}, \textsc{Spar-Sink} can be combined with the inexact proximal point method to approximate unregularized OT and UOT distances; further analyses are left to our future work.
To handle potential outliers in practical applications, we also plan to extend the \textsc{Spar-Sink} method to the robust optimal transport framework \citep{le2021robust} in the future.



\acks{We thank the anonymous reviewers and Action Editor Michael Mahoney for their constructive comments that improved the quality of this paper. We also thank members of the Big Data Analytics Lab at the University of Georgia for their helpful comments. 
The authors would like to acknowledge the support from Beijing Municipal Natural Science Foundation No. 1232019, National Natural Science Foundation of China Grant No. 12101606, No. 12001042, No. 12271522, Renmin University of China research fund program for young scholars, and Beijing Institute of Technology research fund program for young scholars. Mengyu Li is supported by the Outstanding Innovative Talents Cultivation Funded Programs 2021 of Renmin University of China. 
The authors report there are no competing interests to declare.}


\newpage

\appendix
\section{Importance Sparsification for Wasserstein Barycenters}
In this section, we extend the \textsc{Spar-Sink} method to approximate fixed-support Wasserstein barycenters, which have been widely used in the machine learning community \citep{rabin2011wasserstein, benamou2015iterative, montesuma2021wasserstein}.   

\subsection{Wasserstein Barycenters}
Given a set of probability measures $\{\b_1,\ldots,\b_m\} \subset \Delta^{n-1}$ and weights $\bm{w}\in\Delta^{m-1}$, a Wasserstein barycenter is computed by
\begin{equation}\label{eq:wb}
    \min_{\bm{q} \in \Delta^{n-1}} \sum_{k=1}^m w_k \OT(\bm{q}, \b_k),
\end{equation}
where $\OT(\bm{q}, \b_k)$ is defined in \eqref{eq:ot}, associated with a prespecified distance matrix $\C_k \in \RR_{+}^{n\times n}$ of the power $p$, for $k\in [m]$.
Following the success of \cite{cuturi2013sinkhorn}, the solution to~\eqref{eq:wb} can also be approximated via entropic smoothing \citep{cuturi2014fast}; that is, replacing $\OT(\cdot)$ with $\OT_\eps(\cdot)$ defined in \eqref{eq:rot} and leading to
\begin{equation}\label{eq:rwb}
    \bm{q}_\eps^\ast := \arg\min_{\bm{q} \in \Delta^{n-1}} \sum_{k=1}^m w_k \OT_\eps(\bm{q}, \b_k).
\end{equation}
By introducing kernel matrices $\K_k := \exp(-\C_k/\eps)$, the problem~\eqref{eq:rwb} can be rewritten as a weighted KL projection problem,
\begin{equation*}
    \min _{\T_1,\ldots,\T_m\in \RR_{+}^{n\times n}} \sum_{k=1}^m w_k \eps \KL (\T_k \| \K_k) \quad \text{s.t.}~\T_k^{\top} \mathbf{1}_{n} = \b_k, k\in[m]~\text{and}~\T_1 \mathbf{1}_{n} = \cdots = \T_m \mathbf{1}_{n}.
\end{equation*}
Here, the barycenter $\bm{q}$ is implied in the row marginals of transport plans as $\T_k \mathbf{1}_{n}=\bm{q}$ for $k\in[m]$.
The authors of \cite{benamou2015iterative} proposed an iterative Bregman projection (IBP) algorithm, shown in Algorithm~\ref{alg:IBP}, to solve \eqref{eq:rwb} effectively. In Algorithm~\ref{alg:IBP}, the notations $\odot$ and $\oslash$ represent element-wise multiplication and division, respectively.

\begin{algorithm}
\caption{\textsc{IBP}($\{\K_k\}_{k=1}^m, \{\b_k\}_{k=1}^m, \bm{w}, \delta$)}
\begin{algorithmic}[1]
\State {\bf Initialize:}
$t \leftarrow 0; \bm{q}^{(0)} \leftarrow \mathbf{1}_n/n; \u^{(0)}_k \leftarrow \mathbf{1}_n,~\text{for}~k\in[m]$
\Repeat
    \State $t \leftarrow t+1$
    \State \textbf{for} $k=1 \textbf{ to } m$: $~\v^{(t)}_k \leftarrow \b_k \oslash \K^{\top}_k \u^{(t-1)}_k; \quad \u^{(t)}_k \leftarrow \bm{q}^{(t-1)} \oslash \K_k \v^{(t)}_k$
    \State $\bm{q}^{(t)} \leftarrow (\K_1 \v_1^{(t)})^{w_1} \odot \cdots \odot (\K_m \v_m^{(t)})^{w_m}$
\Until{$\|\bm q^{(t)}-\bm q^{(t-1)}\|_1 \le \delta$}
\State {\bf Output:} $\bm{q}^{(t)}$
\end{algorithmic}
\label{alg:IBP}
\end{algorithm}

\subsection{Proposed Algorithm}
As a generalized Sinkhorn algorithm, the IBP algorithm needs to compute matrix-vector multiplications w.r.t. $\K_1, \ldots, \K_m$ at each iteration.
Analogous to the idea of \textsc{Spar-Sink}, we approximate the dense kernel matrices with sparse sketches $\tK_1, \ldots, \tK_m$ and propose a new \textsc{Spar-IBP} algorithm.

Recall that $\tK_k$ is defined by
\begin{align}\label{eq:pois-k}
\widetilde{K}_{k, i j}= \begin{cases} K_{k, i j}/p^\ast_{k, i j} & \text { with prob. } p^\ast_{k, i j}=\min \left(1, s p_{k, i j}\right) \\
0 & \text { otherwise. }\end{cases}  
\end{align}
According to the principle of importance sampling, the sampling probability $p_{k,ij}$ should be proportional to $\sqrt{q^\ast_{\eps,i} b_{k,j}}$. Unfortunately, such a probability depends on the unknown barycenter. To bypass the obstacle, we propose to replace the unknown $\bm{q}_\eps^\ast$ with the initial value $\bm{q}^{(0)} = \mathbf{1}_n/n$, which implies the elements in the same column of $\tK_k$ have the equal probability to be selected. Such a procedure is reasonable considering it is common that the prior information of the barycenter is inaccessible. 

\begin{algorithm}
\caption{\textsc{Spar-IBP} algorithm for Wasserstein barycenters}
\begin{algorithmic}[1]
\State {\bf Input:}
$\{\K_k\}_{k=1}^m \subset \RR_{+}^{n\times n}$, $\{\b_k\}_{k=1}^m \subset \Delta^{n-1}$, $\bm{w}\in \Delta^{m-1}$, $0<s<n^2$, $\delta > 0$
\State For $k\in [m]$, construct $\tK_k$ according to~\eqref{eq:pois-k} with 
$$
p_{k, ij} = \frac{\sqrt{b_{k,j}}}{n\sum_{j=1}^n \sqrt{b_{k,j}}}, \quad 1\leq i,j\leq n
$$
\State Compute $\widetilde{\bm q}_\eps^\ast = \textsc{IBP}(\{\tK_k\}_{k=1}^m, \{\b_k\}_{k=1}^m, \bm{w}, \delta)$ by using Algorithm~\ref{alg:IBP}
\State {\bf Output:} $\widetilde{\bm q}_\eps^\ast$
\end{algorithmic}
\label{alg:spar-IBP}
\end{algorithm}

Algorithm~\ref{alg:spar-IBP} details the proposed \textsc{Spar-IBP} algorithm for approximating Wasserstein barycenters. 
Compared to Algorithm~\ref{alg:IBP}, it reduces the computational complexity of each iteration from $O(mn^2)$ to $O(ms)$.

\section{Technical Details}
In this appendix, we provide technical details of the theoretical results stated within the manuscript.

\subsection{Proof of Theorem~\ref{thm:sot}}
Recall that Sinkhorn algorithm aims to solve the following optimization problem
\begin{equation}\label{eq:reg-ot}
\T_\eps^\ast = \min_{\T \in \U(\a, \b)} \langle \T,\C\rangle-\eps H(\T).
\end{equation}
The dual problem of \eqref{eq:reg-ot} is
\begin{equation}\label{eq:dual-ot}
\max_{\boldsymbol{\alpha}\in \RR^{n},\boldsymbol{\beta}\in \RR^{n}} f(\boldsymbol{\alpha},\boldsymbol{\beta}) := \a^\top\boldsymbol{\alpha} + \b^\top\boldsymbol{\beta} - \eps(e^{\boldsymbol{\alpha}/\eps})^\top\K e^{\boldsymbol{\beta}/\eps} + \eps,
\end{equation}
where $\K=\exp(-\C/\eps)$ is the kernel matrix, and $\boldsymbol{\alpha}, \boldsymbol{\beta} \in \RR^{n}$ are dual variables. As been defined in Section~\ref{sec:sparsink-ot}, $\tT_\eps^\ast$ is the sparsification counterpart to \eqref{eq:reg-ot}, and the corresponding dual problem becomes
\begin{equation}\label{eq:dual-ot-spar}
\max_{\boldsymbol{\alpha}\in \RR^{n},\boldsymbol{\beta}\in \RR^{n}} \tilde{f}(\boldsymbol{\alpha},\boldsymbol{\beta}) := \a^\top\boldsymbol{\alpha} + \b^\top\boldsymbol{\beta} - \eps(e^{\boldsymbol{\alpha}/\eps})^\top\tK e^{\boldsymbol{\beta}/\eps} + \eps,
\end{equation}
which replaces $\K$ in \eqref{eq:dual-ot} with its sparse sketch $\tK$.

To prove the Theorem~\ref{thm:sot}, we first introduce several lemmas.

\begin{lemma}\label{lem:a1}
    Suppose both $\K$ and $\tK$ are positive definite. Further suppose the condition number of $\K$ and $\tK$ are bounded by $c_2$ and $c_2^\prime$, respectively, and there exist constants $c_a,c_{\alpha}>0$ such that
    $n\|\a\|_\infty\le c_a$, $\operatorname{osc}(\boldsymbol{\alpha}^*)\le c_{\alpha}\eps$, $\operatorname{osc}(\bar{\boldsymbol{\alpha}})\le c_{\alpha}\eps$.
    Let $(\boldsymbol{\alpha}^*,\boldsymbol{\beta}^*)$ be the solution to \eqref{eq:dual-ot}, and $(\bar{\boldsymbol{\alpha}},\bar{\boldsymbol{\beta}})$ be the solution to \eqref{eq:dual-ot-spar}. It follows that
    \begin{equation}\label{eq:lem1}
    |f(\boldsymbol{\alpha}^*,\boldsymbol{\beta}^*)-f(\bar{\boldsymbol{\alpha}},\bar{\boldsymbol{\beta}})|\le \eps \sqrt{c_a}e^{c_{\alpha}}\left(c_2+c_2^\prime \left|1-\frac{\|\tK-\K\|_2}{\|\K\|_2}\right|^{-1}\right)\frac{\|\tK-\K\|_2}{\|\K\|_2}.
    \end{equation}
    where $\|\cdot\|_2$ denotes the spectral norm (i.e., the maximal singular value) of a matrix, $\|\cdot\|_\infty$ denotes the infinity norm of a vector, and $\|\cdot\|_*$ denotes the nuclear norm.
\end{lemma}

\begin{proof}
    First, we establish the following inequality:
    \begin{equation}\label{eq:a16}
    |f(\boldsymbol{\alpha}^*,\boldsymbol{\beta}^*)-f(\bar{\boldsymbol{\alpha}},\bar{\boldsymbol{\beta}})|\le|f(\boldsymbol{\alpha}^*,\boldsymbol{\beta}^*)-\tilde{f}({\boldsymbol{\alpha}}^*,{\boldsymbol{\beta}}^*)|+|\tilde{f}(\bar{\boldsymbol{\alpha}},\bar{\boldsymbol{\beta}})-f(\bar{\boldsymbol{\alpha}},\bar{\boldsymbol{\beta}})|.
    \end{equation}
    By the definitions of $\boldsymbol{\alpha}^*,\boldsymbol{\beta}^*,\bar{\boldsymbol{\alpha}},\bar{\boldsymbol{\beta}}$, it holds that 
    $$\tilde{f}(\bar{\boldsymbol{\alpha}},\bar{\boldsymbol{\beta}}) \ge \tilde{f}(\boldsymbol{\alpha}^*,\boldsymbol{\beta}^*), \quad f(\boldsymbol{\alpha}^*,\boldsymbol{\beta}^*) \ge f(\bar{\boldsymbol{\alpha}},\bar{\boldsymbol{\beta}}).$$ 
    We consider the following two cases:
    \begin{itemize}
        \item[] $\text{Case 1. } f(\boldsymbol{\alpha}^*,\boldsymbol{\beta}^*) \ge \tilde{f}(\bar{\boldsymbol{\alpha}},\bar{\boldsymbol{\beta}})$;
        \item[] $\text{Case 2. } f(\boldsymbol{\alpha}^*,\boldsymbol{\beta}^*) < \tilde{f}(\bar{\boldsymbol{\alpha}},\bar{\boldsymbol{\beta}})$.
    \end{itemize}
    For Case 1, it holds that $0\le f(\boldsymbol{\alpha}^*,\boldsymbol{\beta}^*) - \tilde{f}(\bar{\boldsymbol{\alpha}},\bar{\boldsymbol{\beta}}) \le f(\boldsymbol{\alpha}^*,\boldsymbol{\beta}^*) - \tilde{f}(\boldsymbol{\alpha}^*,\boldsymbol{\beta}^*)$, and thus $|f(\boldsymbol{\alpha}^*,\boldsymbol{\beta}^*) - \tilde{f}(\bar{\boldsymbol{\alpha}},\bar{\boldsymbol{\beta}})| \le |f(\boldsymbol{\alpha}^*,\boldsymbol{\beta}^*) - \tilde{f}(\boldsymbol{\alpha}^*,\boldsymbol{\beta}^*)|$, which leads to~\eqref{eq:a16} by combining the triangle inequality
    $$|f(\boldsymbol{\alpha}^*,\boldsymbol{\beta}^*)-f(\bar{\boldsymbol{\alpha}},\bar{\boldsymbol{\beta}})| \le |f(\boldsymbol{\alpha}^*,\boldsymbol{\beta}^*)-\tilde{f}(\bar{\boldsymbol{\alpha}},\bar{\boldsymbol{\beta}})|+|\tilde{f}(\bar{\boldsymbol{\alpha}},\bar{\boldsymbol{\beta}})-f(\bar{\boldsymbol{\alpha}},\bar{\boldsymbol{\beta}})|.$$
    For Case 2, (i) when $f(\bar{\boldsymbol{\alpha}},\bar{\boldsymbol{\beta}}) \le \tilde{f}({\boldsymbol{\alpha}}^*,{\boldsymbol{\beta}}^*)$, it holds that $0\le \tilde{f}({\boldsymbol{\alpha}}^*,{\boldsymbol{\beta}}^*) - f(\bar{\boldsymbol{\alpha}},\bar{\boldsymbol{\beta}}) \le \tilde{f}(\bar{\boldsymbol{\alpha}},\bar{\boldsymbol{\beta}}) - f(\bar{\boldsymbol{\alpha}},\bar{\boldsymbol{\beta}})$, and thus $| \tilde{f}({\boldsymbol{\alpha}}^*,{\boldsymbol{\beta}}^*) - f(\bar{\boldsymbol{\alpha}},\bar{\boldsymbol{\beta}})| \le |\tilde{f}(\bar{\boldsymbol{\alpha}},\bar{\boldsymbol{\beta}}) - f(\bar{\boldsymbol{\alpha}},\bar{\boldsymbol{\beta}})|$, which leads to~\eqref{eq:a16} by combining the triangle inequality
    $$|f(\boldsymbol{\alpha}^*,\boldsymbol{\beta}^*)-f(\bar{\boldsymbol{\alpha}},\bar{\boldsymbol{\beta}})| \le |f(\boldsymbol{\alpha}^*,\boldsymbol{\beta}^*)-\tilde{f}(\boldsymbol{\alpha}^*,\boldsymbol{\beta}^*)|+|\tilde{f}(\boldsymbol{\alpha}^*,\boldsymbol{\beta}^*)-f(\bar{\boldsymbol{\alpha}},\bar{\boldsymbol{\beta}})|.$$
    (ii) When $f(\bar{\boldsymbol{\alpha}},\bar{\boldsymbol{\beta}}) > \tilde{f}({\boldsymbol{\alpha}}^*,{\boldsymbol{\beta}}^*)$, we have $|f(\boldsymbol{\alpha}^*,\boldsymbol{\beta}^*)-f(\bar{\boldsymbol{\alpha}},\bar{\boldsymbol{\beta}})|\le|f(\boldsymbol{\alpha}^*,\boldsymbol{\beta}^*)-\tilde{f}({\boldsymbol{\alpha}}^*,{\boldsymbol{\beta}}^*)|$; then~\eqref{eq:a16} establishes because $|\tilde{f}(\bar{\boldsymbol{\alpha}},\bar{\boldsymbol{\beta}})-f(\bar{\boldsymbol{\alpha}},\bar{\boldsymbol{\beta}})| \ge 0$.
    
    Consequently, we conclude the inequality~\eqref{eq:a16} by combining Cases 1 and 2. 

    Let $\boldsymbol{u}^*=e^{\boldsymbol{\alpha}^*/\eps}$, $\boldsymbol{v}^*=e^{\boldsymbol{\beta}^*/\eps}$, $\bar{\boldsymbol{u}}=e^{\bar{\boldsymbol{\alpha}}/\eps}$, and $\bar{\boldsymbol{v}}=e^{\bar{\boldsymbol{\beta}}/\eps}$. By the first-order optimality conditions for~\eqref{eq:dual-ot} and~\eqref{eq:dual-ot-spar},
    \[
    \boldsymbol{u}^*\odot (\K \boldsymbol{v}^*)=\a,
    \qquad
    \bar{\boldsymbol{u}}\odot (\tK \bar{\boldsymbol{v}})=\a.
    \]
    Hence,
    \begin{align*}
    \|\K\boldsymbol{v}^*(\boldsymbol{u}^*)^\top\|_*
    &= \|\K\boldsymbol{v}^*\|_2\,\|\boldsymbol{u}^*\|_2 \\
    &\le \|\a\oslash \boldsymbol{u}^*\|_2\,\sqrt{n}\,\|\boldsymbol{u}^*\|_\infty \\
    &\le \sqrt{n}\,\|\a\|_2\,\| (\boldsymbol{u}^*)^{-1}\|_\infty\,\|\boldsymbol{u}^*\|_\infty \\
    &\le \sqrt{n\|\a\|_\infty}\,\exp\!\left(\operatorname{osc}(\boldsymbol{\alpha}^*)/\eps\right)
    \le \sqrt{c_a}e^{c_{\alpha}}.
    \end{align*}
    Applying the same argument to $(\bar{\boldsymbol{\alpha}},\bar{\boldsymbol{\beta}})$ gives
    \[
    \|\tK\bar{\boldsymbol{v}}\bar{\boldsymbol{u}}^\top\|_*\le \sqrt{c_a}e^{c_{\alpha}}.
    \]
    For notational simplicity, write $c_H:=\sqrt{c_a}e^{c_{\alpha}}$.

    Next, we provide an upper bound for the right-hand side of~\eqref{eq:a16}. Simple calculation yields that
	\begin{align}
        \nonumber |f(\boldsymbol{\alpha}^*,\boldsymbol{\beta}^*)-\tilde{f}(\boldsymbol{\alpha}^*,\boldsymbol{\beta}^*)|
	 =&|\eps\langle e^{\boldsymbol{\alpha}^*/\eps},(\tK-\K)e^{\boldsymbol{\beta}^*/\eps}\rangle| \\
	 =&\eps |\textrm{tr}\{(e^{\boldsymbol{\alpha}^*/\eps})^\top(\tK-\K) \K^{-1}\K e^{\boldsymbol{\beta}^*/\eps}\}|. \label{eq:lem1-f1}
	\end{align}
	Moreover, one can find that
    \begin{align*}
        \|(\tK-\K)\K^{-1}\|_2 \le \|\tK-\K\|_2 / \lambda_{\min}(\K),
    \end{align*}
    where $\lambda_{\min}(\K)$ is the minimal eigenvalue of $\K$.
    For notation simplicity, denote $\mathbf{G}=(\tK-\K) \K^{-1}$ and $\mathbf{H}=\K e^{\boldsymbol{\beta}^*/\eps}(e^{\boldsymbol{\alpha}^*/\eps})^\top$.
	By H\"older's inequality for the trace duality between the spectral norm and the nuclear norm,
	\begin{align*}
	|\textrm{tr}(\mathbf{G} \mathbf{H})|\le \|\mathbf{G}\|_2\|\mathbf{H}\|_*.
	\end{align*}
	Therefore, \eqref{eq:lem1-f1} can be bounded by
    \begin{align}
        \nonumber |f(\boldsymbol{\alpha}^*,\boldsymbol{\beta}^*)-\tilde{f}(\boldsymbol{\alpha}^*,\boldsymbol{\beta}^*)| &\leq \eps \|(\tK-\K)\K^{-1}\|_2 \|\K e^{\boldsymbol{\beta}^*/\eps}(e^{\boldsymbol{\alpha}^*/\eps})^\top\|_*  \\
        \nonumber &\leq \eps c_H \|\tK-\K\|_2/\lambda_{\min}(\K) \\
        &\leq \eps c_H c_2 \|\tK-\K\|_2/\|\K\|_2. \label{eq:a30}
    \end{align}
    
	Using the same procedure, we obtain that
	\begin{align}
	\nonumber |f({\boldsymbol{\bar\alpha}},\boldsymbol{\bar\beta})-\tilde{f}(\boldsymbol{\bar\alpha},\boldsymbol{\bar\beta})|
	\nonumber =&|\eps\langle e^{{\boldsymbol{\bar\alpha}}/\eps},(\tK-\K)e^{{\boldsymbol{\bar\beta}}/\eps}\rangle| \\
        \nonumber =& \eps |\langle e^{{\boldsymbol{\bar\alpha}}/\eps},(\tK-\K) \tK^{-1}\tK e^{{\boldsymbol{\bar\beta}}/\eps}\rangle| \\
        \le & \eps \|(\tK-\K)\tK^{-1}\|_2 \|\tK e^{\bar{\boldsymbol{\beta}}/\eps}(e^{\bar{\boldsymbol{\alpha}}/\eps})^\top\|_*. \label{eq:lem1-f2}
	\end{align}
	Furthermore, simple calculation yields that 
	\begin{align*}
	\|(\tK-\K)\tK^{-1}\|_2 &\le \|\tK-\K\|_2/\lambda_{\min}(\tK)\\
        &\le c_2^\prime \|\tK-\K\|_2/\|\tK\|_2\\
	&= c_2^\prime \frac{\|\tK-\K\|_2}{\|\K\|_2}\frac{\|\K\|_2}{\|\tK\|_2}\\
	&\le c_2^\prime \frac{\|\tK-\K\|_2}{\|\K\|_2}\frac{\|\K\|_2}{|\|\K\|_2-\|\tK-\K\|_2|}\\
	&= c_2^\prime \frac{\|\tK-\K\|_2}{\|\K\|_2}\left|1-\frac{\|\tK-\K\|_2}{\|\K\|_2}\right|^{-1},
	\end{align*}
	where the last inequality comes from the triangle inequality. 
	Therefore, \eqref{eq:lem1-f2} satisfies that
	\begin{align}\label{eq:a31}
	|f(\bar{\boldsymbol{\alpha}},\bar{\boldsymbol{\beta}})-\tilde{f}(\bar{\boldsymbol{\alpha}},\bar{\boldsymbol{\beta}})| \le \eps c_H c_2^\prime \frac{\|\tK-\K\|_2}{\|\K\|_2}\left|1-\frac{\|\tK-\K\|_2}{\|\K\|_2}\right|^{-1}.
	\end{align}
	
	Combining \eqref{eq:a16}, \eqref{eq:a30}, and \eqref{eq:a31}, the result follows.
\end{proof}

Now we show that under some mild conditions, our subsampling procedure yields a relatively small difference between $\tK$ and $\K$.

\begin{lemma}\label{lem:a2}
	Suppose the regularity conditions (i)---(iii) in Theorem~\ref{thm:sot} hold.
	For any $\epsilon>0$ and $n>76$, we have
	\begin{equation}\label{eq:lem4-a19}
	\mathbb{P}\left(\frac{\|\tK-\K\|_2}{\|\K\|_2}\ge {2\sqrt{2}(2+\epsilon)}c_1\sqrt{\frac{n^{3-2\alpha}}{c_3 s}}\right)<2\exp\left(-\frac{16}{\epsilon^4} \log^4(n) \right).
	\end{equation}
\end{lemma}

\begin{proof}
	By the definition of $\K$, one can find that $K_{ij}\le 1$ for any $i,j=1,\ldots,n$. Simple calculation yields that
	\begin{align*}
	\mathbb{E}\left(\|\K\|_2^{-1}\widetilde{K}_{ij}\right) &= \|\K\|_2^{-1}K_{ij},\\
	\textrm{Var}\left(\|\K\|_2^{-1}\widetilde{K}_{ij}\right) &< \frac{K_{ij}^2}{p_{ij}^{*}\|\K\|_2^2}\le \frac{1}{p_{ij}^{*}\|\K\|_2^2}\le \frac{n^{2}}{c_3 s\|\K\|_2^2}.
	\end{align*}
	Also note that $\|\K\|_2^{-1}\widetilde{K}_{ij}$ lies between $0$ and $(p_{ij}^*\|\K\|_2)^{-1}$ for any $(i,j)$th entry. 
	Thus, $\|\K\|_2^{-1}\widetilde{K}_{ij}$ takes the value in an interval of length not larger than $L$, with
        \begin{align*}
        	L:=\frac{n^{2}}{c_3 s \|\K\|_2} &\le \sqrt{\frac{n^{3-2\alpha}}{2c_3 s}} \times \sqrt{\frac{n^{2}}{c_3 s\|\K\|_2^2}} \times \sqrt{2n}\\
         &\le \left(\frac{\log(1+\epsilon)}{2\log(2n)}\right)^2 \times \sqrt{\frac{n^{2}}{c_3 s\|\K\|_2^2}} \times \sqrt{2n},
        \end{align*}
        where $L$ is defined according to the condition~(ii), and the last inequality holds under the condition~(iii).
	Therefore, by applying Theorem 3.1 in \cite{achlioptas2007fast} on $\|\K\|_2^{-1}\tK$, we have
	\begin{equation}
	\mathbb{P}\left(\frac{\|\tK-\K\|_2}{\|\K\|_2}\ge 2(2+\epsilon)\sqrt{\frac{{2n^3}}{c_3 s \|\K\|_2^2}}\right)<2\exp\left(-\frac{16}{\epsilon^4} \log^4(n) \right).\label{eq:a32}
	\end{equation}
	Further, combining \eqref{eq:a32} with the condition~(i) results in the inequality~\eqref{eq:lem4-a19}.
\end{proof}

Finally, we prove the {\bf Theorem~\ref{thm:sot}} in the manuscript.

\bigskip
\begin{proof}
	From Lemma~\ref{lem:a2}, it is straightforward to see that $\|\K\|_2^{-1}\tK \to \|\K\|_2^{-1}\K$ in probability. Thus, $\tK$ tends to be positive definite since $\K$ is a positive definite kernel matrix, and it holds that $c_2^\prime \to c_2$.
	Note that $n^{3-2\alpha}/s\to 0$ as $n\to \infty$, it is easy to see that when $n$ is large enough, we have
    $$c^\prime \sqrt{n^{3-2\alpha}/s}\le 1/2~~\text{with}~~c^\prime = 2\sqrt{2}(2+\epsilon)c_1/\sqrt{c_3},$$ 
    and this implies $(1+|1-c^\prime\sqrt{n^{3-2\alpha}/s}|^{-1})\le 3$.
    Combining Lemmas~\ref{lem:a1} and~\ref{lem:a2}, the result follows.
\end{proof}

\subsection{Proof of Theorem~\ref{thm:suot}}

Now we focus on the entropy-regularized UOT problem, whose dual problem is
\begin{equation}\label{eq:dual-uot}
\max_{\boldsymbol{\alpha}\in \RR^{n},\boldsymbol{\beta}\in \RR^{n}} f_u(\boldsymbol{\alpha},\boldsymbol{\beta}) := \a^\top \one-\lambda \a^\top e^{-\boldsymbol{\alpha}/\lambda} + \b^\top\one -\lambda\b^\top  e^{-\boldsymbol{\beta}/\lambda} - \eps(e^{\boldsymbol{\alpha}/\eps})^\top\K e^{\boldsymbol{\beta}/\eps}.
\end{equation}
Let
\begin{equation}\label{eq:dual-uot-spar}
\max_{\boldsymbol{\alpha}\in \RR^{n},\boldsymbol{\beta}\in \RR^{n}} \tilde{f}_u(\boldsymbol{\alpha},\boldsymbol{\beta}) := \a^\top \one-\lambda \a^\top e^{-\boldsymbol{\alpha}/\lambda} + \b^\top\one -\lambda\b^\top  e^{-\boldsymbol{\beta}/\lambda} - \eps(e^{\boldsymbol{\alpha}/\eps})^\top\tK e^{\boldsymbol{\beta}/\eps}
\end{equation}
be the sparsification counterpart to \eqref{eq:dual-uot}, which replaces $\K$ in \eqref{eq:dual-uot} with its sparse sketch $\tK$.
Apparently, it is the dual problem for the entropic UOT problem with kernel $\tK$.
The following lemma holds by similar procedures as in Lemma~\ref{lem:a1}.

\begin{lemma}\label{lem:a3}
	Suppose the regularity condition~(v) in Theorem~\ref{thm:suot} hold.
	Further suppose both $\K$ and $\tK$ are positive definite, their condition numbers are respectively bounded by $c_2$ and $c_2^\prime$, and there exist constants $c_a,c_{u,\alpha}>0$ such that
    $n\|\a\|_\infty\le c_a$, $\|\boldsymbol{\alpha}^*\|_\infty\le c_{u,\alpha}\eps$, $\|\bar{\boldsymbol{\alpha}}\|_\infty\le c_{u,\alpha}\eps$.
	Let $(\boldsymbol{\alpha}^*,\boldsymbol{\beta}^*)$ be the solution to \eqref{eq:dual-uot}, and $(\bar{\boldsymbol{\alpha}},\bar{\boldsymbol{\beta}})$ be the solution to \eqref{eq:dual-uot-spar}. It follows that
	\begin{equation*}
	|f_u(\boldsymbol{\alpha}^*,\boldsymbol{\beta}^*)-f_u(\bar{\boldsymbol{\alpha}},\bar{\boldsymbol{\beta}})|\le \eps \sqrt{c_a}e^{(2+c_6)c_{u,\alpha}}\left(c_2+c_2^\prime \left|1-\frac{\|\tK-\K\|_2}{\|\K\|_2}\right|^{-1}\right)\frac{\|\tK-\K\|_2}{\|\K\|_2}.
	\end{equation*}
\end{lemma}

\begin{proof}
    The proof is similar to that of Lemma~\ref{lem:a1}.
    Let $\boldsymbol{u}^*=e^{\boldsymbol{\alpha}^*/\eps}$, $\boldsymbol{v}^*=e^{\boldsymbol{\beta}^*/\eps}$, $\bar{\boldsymbol{u}}=e^{\bar{\boldsymbol{\alpha}}/\eps}$, and $\bar{\boldsymbol{v}}=e^{\bar{\boldsymbol{\beta}}/\eps}$. By the first-order optimality conditions for~\eqref{eq:dual-uot} and~\eqref{eq:dual-uot-spar},
    \[
    (\boldsymbol{u}^*)^{(\lambda+\eps)/\lambda}\odot (\K \boldsymbol{v}^*)=\a,
    \qquad
    \bar{\boldsymbol{u}}^{(\lambda+\eps)/\lambda}\odot (\tK \bar{\boldsymbol{v}})=\a.
    \]
    Hence,
    \begin{align*}
    \|\K\boldsymbol{v}^*(\boldsymbol{u}^*)^\top\|_*
    &= \|\K\boldsymbol{v}^*\|_2\,\|\boldsymbol{u}^*\|_2 \\
    &\le \|\a\oslash (\boldsymbol{u}^*)^{(\lambda+\eps)/\lambda}\|_2\,\sqrt{n}\,\|\boldsymbol{u}^*\|_\infty \\
    &\le \sqrt{n}\,\|\a\|_2\,\|(\boldsymbol{u}^*)^{-1}\|_\infty^{(\lambda+\eps)/\lambda}\,\|\boldsymbol{u}^*\|_\infty \\
    &\le \sqrt{n\|\a\|_\infty}\,\exp\!\left(\left(2+\frac{\eps}{\lambda}\right)c_{u,\alpha}\right)
    \le \sqrt{c_a}e^{(2+c_6)c_{u,\alpha}}.
    \end{align*}
    Applying the same argument to $(\bar{\boldsymbol{\alpha}},\bar{\boldsymbol{\beta}})$ gives
    \[
    \|\tK\bar{\boldsymbol{v}}\bar{\boldsymbol{u}}^\top\|_*\le \sqrt{c_a}e^{(2+c_6)c_{u,\alpha}}.
    \]
    For notational simplicity, write $c_{H,u}:=\sqrt{c_a}e^{(2+c_6)c_{u,\alpha}}$.
	By the definitions of $\boldsymbol{\alpha}^*,\boldsymbol{\beta}^*,\bar{\boldsymbol{\alpha}},\bar{\boldsymbol{\beta}}$ and the triangle inequality, it holds that
	\begin{equation}
	|f_u(\boldsymbol{\alpha}^*,\boldsymbol{\beta}^*)-f_u(\bar{\boldsymbol{\alpha}},\bar{\boldsymbol{\beta}})| \le |f_u(\boldsymbol{\alpha}^*,\boldsymbol{\beta}^*)-\tilde{f}_u({\boldsymbol{\alpha}}^*,{\boldsymbol{\beta}}^*)|+|\tilde{f}_u(\bar{\boldsymbol{\alpha}},\bar{\boldsymbol{\beta}})-f_u(\bar{\boldsymbol{\alpha}},\bar{\boldsymbol{\beta}})|. \label{eq:lem6-27}
	\end{equation}
	For the first term in the right-hand side of \eqref{eq:lem6-27}, simple calculation yields that
	\begin{align}
	|f_u(\boldsymbol{\alpha}^*,\boldsymbol{\beta}^*)-\tilde{f}_u(\boldsymbol{\alpha}^*,\boldsymbol{\beta}^*)|
	=&|\eps\langle e^{\boldsymbol{\alpha}^*/\eps},(\tK-\K)e^{\boldsymbol{\beta}^*/\eps}\rangle| \nonumber\\
	\le& \eps \|(\tK-\K)\K^{-1}\|_2 \|\K e^{\boldsymbol{\beta}^*/\eps}(e^{\boldsymbol{\alpha}^*/\eps})^\top\|_* \label{eq:lem3-22}\\
        \le& \eps c_{H,u} \|(\tK-\K)\K^{-1}\|_2 \nonumber\\
        \le& \eps c_{H,u} c_2 \frac{\|\tK-\K\|_2}{\|\K\|_2}. \nonumber
	\end{align}
	
	Applying the same techniques to Lemma~\ref{lem:a1}, we conclude that
	\begin{equation}\label{eq:lem6-29}
	|f_u(\bar{\boldsymbol{\alpha}},\bar{\boldsymbol{\beta}})-\tilde{f}_u(\bar{\boldsymbol{\alpha}},\bar{\boldsymbol{\beta}})| \le \eps c_{H,u} c_2^\prime \frac{\|\tK-\K\|_2}{\|\K\|_2}\left|1-\frac{\|\tK-\K\|_2}{\|\K\|_2}\right|^{-1}.
	\end{equation}
	Combining~\eqref{eq:lem6-27}, \eqref{eq:lem3-22}, and \eqref{eq:lem6-29} leads to the inequality in Lemma~\ref{lem:a3}.
\end{proof}

By the same argument as in the proof of Theorem~\ref{thm:sot}, $\tK$ is positive definite and its condition number is bounded with probability approaching one; hence Lemma~\ref{lem:a3} applies. Then, Theorem~\ref{thm:suot} is a direct result of Lemmas~\ref{lem:a2} and~\ref{lem:a3}.

\subsection{Proof of Theorem~\ref{thm:time}}

First, we provide a lemma that is used to establish the iteration bound for Algorithm~\ref{alg:core-uot}, i.e., \textsc{Spar-Sink} for UOT.

\begin{lemma}\label{lem:uot-time}
Suppose that $(\bar{\boldsymbol{\alpha}},\bar{\boldsymbol{\beta}})$ is the solution to \eqref{eq:dual-uot-spar}. Under the conditions of Theorem~\ref{thm:suot}, the infinity norms of $\bar{\boldsymbol{\alpha}}$ and $\bar{\boldsymbol{\beta}}$ are bounded by
$$
\max \left\{\|\bar{\boldsymbol{\alpha}}\|_{\infty}, \|\bar{\boldsymbol{\beta}}\|_{\infty}\right\} \leq \lambda R^\prime.
$$
Here, $R^\prime = \max \{\|\log (\a)\|_{\infty},\|\log (\b)\|_{\infty}\} +\log (n) + \max \{\log (n) + c_9, \|\C\|_{\infty}/\eps\}$, where $c_9$ is a constant only depending on $c_3$ and $c_4$.
\end{lemma}

\begin{proof}
This proof follows the proof of Lemma~3 in \cite{pham2020unbalanced}. According to Lemma 1 in \cite{pham2020unbalanced}, it holds that
\begin{equation}\label{eq:lem7-1}
    \frac{\bar{\alpha}_i}{\lam} = \log (a_i)-\log \left(\sum\nolimits_{j=1}^n e^{(\bar{\alpha}_i+\bar{\beta}_j)/\eps} \widetilde{K}_{ij}\right).
\end{equation}
Denote $\mathcal{S}=\{(i,j)\in[n]\times [n] | \widetilde{K}_{ij}>0\}$ and $\mathcal{S}_i=\{j\in[n] | \widetilde{K}_{ij}>0\}$ for $i\in [n]$.
By introducing a matrix $\tC\in \RR^{n\times n}$ with
\begin{align*}
\widetilde{C}_{i j}= \begin{cases} C_{i j} + \eps \log (p^\ast_{i j}) & \text { if } (i,j) \in \mathcal{S} \\
0 & \text { otherwise, }\end{cases}  
\end{align*}
then \eqref{eq:lem7-1} can be rewritten as 
\begin{equation*}
    \frac{\bar{\alpha}_i}{\lam} = \log (a_i)-\log \left(\sum\nolimits_{j\in \mathcal{S}_i} e^{(\bar{\alpha}_i+\bar{\beta}_j-\widetilde{C}_{ij})/\eps}\right),
\end{equation*}
which can be further reorganized as
\begin{equation}\label{eq:lem7-2}
    \bar{\alpha}_i\left(\frac{1}{\lam}+\frac{1}{\eps}\right) = \log (a_i)-\log \left(\sum\nolimits_{j\in \mathcal{S}_i} e^{(\bar{\beta}_j-\widetilde{C}_{ij})/\eps}\right).
\end{equation}
According to the properties of the log-sum-exp function, the second term in the right-hand side of~\eqref{eq:lem7-2} has the lower bound
\begin{equation*}
    \log \left(\sum\nolimits_{j\in \mathcal{S}_i} e^{(\bar{\beta}_j-\widetilde{C}_{ij})/\eps}\right) \ge \log (|\mathcal{S}_i|) + \min_{j\in \mathcal{S}_i} \left\{\frac{\bar{\beta}_j-\widetilde{C}_{ij}}{\eps}\right\} \ge -\frac{\|\bar{\bm\beta}\|_{\infty}}{\eps} - \frac{\|\tC\|_{\infty}}{\eps},
\end{equation*}
and it has the upper bound
\begin{equation*}
    \log \left(\sum\nolimits_{j\in \mathcal{S}_i} e^{(\bar{\beta}_j-\widetilde{C}_{ij})/\eps}\right) \le \log (|\mathcal{S}_i|) + \max_{j\in \mathcal{S}_i} \left\{\frac{\bar{\beta}_j-\widetilde{C}_{ij}}{\eps}\right\} \le \log(n) + \frac{\|\bar{\bm\beta}\|_{\infty}}{\eps} +\frac{\|\tC\|_{\infty}}{\eps}.
\end{equation*}
Combining these two bounds together yields that
\begin{equation*}
    \left|\log \left(\sum\nolimits_{j\in \mathcal{S}_i} e^{(\bar{\beta}_j-\widetilde{C}_{ij})/\eps}\right)\right| \le \log(n) + \frac{\|\bar{\bm\beta}\|_{\infty}}{\eps} +\frac{\|\tC\|_{\infty}}{\eps}.
\end{equation*}
Therefore, we have
\begin{equation}\label{eq:lem7-3}
    |\bar{\alpha}_i|\left(\frac{1}{\lam}+\frac{1}{\eps}\right) \le |\log (a_i)| + \log(n) + \frac{\|\bar{\bm\beta}\|_{\infty}}{\eps} +\frac{\|\tC\|_{\infty}}{\eps}.
\end{equation}
By the definition of $\tC$ and conditions (ii)---(iii), we have 
\begin{equation*}
    C_{ij} - \eps \log(n/(c_3 c_4)) \le \widetilde{C}_{ij} \le C_{ij} \quad \text{for } (i,j) \in \mathcal{S},
\end{equation*}
which follows that
\begin{equation*}
    \|\tC\|_\infty \le \max\{\|\C\|_\infty, \eps \log(n/(c_3 c_4))\}.
\end{equation*}
Hence, \eqref{eq:lem7-3} can be further bounded by
\begin{equation*}
    |\bar{\alpha}_i|\left(\frac{1}{\lam}+\frac{1}{\eps}\right) \le |\log (a_i)| + \log(n) + \frac{\|\bar{\bm\beta}\|_{\infty}}{\eps} + \max\left\{\log(n)-\log(c_3 c_4), \frac{\|\C\|_{\infty}}{\eps} \right\}.
\end{equation*}
By choosing an index $i$ such that $|\bar{\alpha}_i| = \|\bar{\bm\alpha}\|_{\infty}$ and noting the fact that $|\log (a_i)| \le \max \{\|\log (\a)\|_{\infty},\|\log (\b)\|_{\infty}\}$, we have
$$
\|\bar{\bm\alpha}\|_{\infty}\left(\frac{1}{\lam}+\frac{1}{\eps}\right) \le \frac{\|\bar{\bm\beta}\|_{\infty}}{\eps}+R^\prime.
$$
Without loss of generality, assume that $\|\bar{\bm\alpha}\|_{\infty} \ge \|\bar{\bm\beta}\|_{\infty}$. Then, the result in Lemma~\ref{lem:uot-time} follows.
\end{proof}

Now, we prove the {\bf Theorem~\ref{thm:time}} in the manuscript.

\bigskip

\begin{proof}
(I) We first show that \textsc{Spar-Sink} has the same order of iterations to Sinkhorn in OT problems.
Suppose $L_1$ (resp. $L_2$) represents the number of iterations in the Sinkhorn algorithm (resp. \textsc{Spar-Sink} algorithm) such that $\|\u^{(t)}-\u^{(t-1)}\|_1 + \|\v^{(t)}-\v^{(t-1)}\|_1 \le \delta$.
From Lemmas 2---4 and the proof of Theorem~2 in \cite{altschuler2017near},
one can conclude that $L_1$ is bounded by
\begin{equation*}
L_1 \le 4\delta^{-2} (\OT_\eps(\a, \b)-f(\bm 0,\bm 0)) \le 4\delta^{-2} \log(q/l)
\end{equation*}
when Algorithm~1 is adopted, where $q=\sum_{i,j}K_{ij}$ and $l=\min_{i,j} K_{ij}$.

Using the same techniques, we have
\begin{equation*}
L_2 \le 4\delta^{-2} (\widetilde{\OT}_\eps(\a,\b)-f(\bm 0,\bm 0))
\end{equation*}
when Algorithm~3 is adopted.
According to Theorem~\ref{thm:sot}, we obtain that
\begin{equation}\label{eq:thm3-proof-1}
\widetilde{\OT}_\eps(\a,\b)-f(\bm 0,\bm 0) = \OT_\eps(\a, \b)-f(\bm 0,\bm 0)+r\le \log(q/l)+r
\end{equation}
with $r=\widetilde{\OT}_\eps(\a,\b)-\OT_\eps(\a, \b)=o_P(1)$ under the regularity conditions in Theorem~\ref{thm:sot}. Hence, we conclude that $L_2 \le O(\delta^{-2} \log(q/l))$ in probability, which has the same order to $L_1$.

(II) Next, we focus on the UOT problems.
Theorem~2 in \cite{pham2020unbalanced} shows that when $\eps=\epsilon/U$ and the number of iterations in the Sinkhorn algorithm achieves
\begin{equation}\label{eq:thm3-proof-2}
    L_1^\prime := 1+\left(\frac{\lam U}{\epsilon}+1\right)\left[\log (8 \eps R)+\log (\lam(\lam+1)) +3 \log \left(\frac{U}{\epsilon}\right)\right],
\end{equation}
the output of Algorithm~\ref{alg:sink-uot} is an $\epsilon$-approximation (see Definition~1 in \cite{pham2020unbalanced} for a detailed definition) of the optimal solution of the UOT problem~\eqref{eq:uot}. The quantities in~\eqref{eq:thm3-proof-2} are defined as follows:
$$
\begin{aligned}
R= & \max \left\{\|\log (\a)\|_{\infty},\|\log (\b)\|_{\infty}\right\} + \max \left\{\log (n), \frac{\|\C\|_{\infty}}{\eps} - \log (n)\right\}, \\
U_1= & \frac{\|\a\|_1+\|\b\|_1}{2}+\frac{1}{2}+\frac{1}{4 \log (n)}, \\
U_2= & \left(\frac{\|\a\|_1+\|\b\|_1}{2}\right)\left[\log \left(\frac{\|\a\|_1+\|\b\|_1}{2}\right)+2 \log (n)-1\right]+\log (n)+\frac{5}{2}, \\
U= & \max \left\{U_1+U_2, 2 \epsilon, \frac{4 \epsilon \log (n)}{\lam}, \frac{4 \epsilon(\|\a\|_1+\|\b\|_1) \log (n)}{\lam}\right\}.
\end{aligned}
$$

By using the same procedures but replacing Lemma~3 in \cite{pham2020unbalanced} with Lemma~\ref{lem:uot-time} above, and combining with the result that $\widetilde{\UOT}_{\lam, \eps}(\a,\b)-\UOT_{\lam, \eps}(\a, \b)=o_P(1)$ under the conditions of Theorem~\ref{thm:suot}, we can conclude that when $\eps=\epsilon/U$ and the number of iterations in the \textsc{Spar-Sink} algorithm achieves
$$
L_2^\prime := 1+\left(\frac{\lam U}{\epsilon}+1\right)\left[\log (8 \eps R^\prime)+\log (\lam(\lam+1)) +3 \log \left(\frac{U}{\epsilon}\right)\right],
$$
the output of Algorithm~\ref{alg:core-uot} is also an $\epsilon$-approximation of the optimal solution of~\eqref{eq:uot} in probability. Due to the fact that $R^\prime = O(R)$, we obtain that $L_1^\prime$ and $L_2^\prime$ are of the same order.

\end{proof}

\section{Additional Numerical Results}

In this appendix, we provide extra experimental results to show the robustness and asymptotic convergence of the proposed algorithm.

\subsection{Sensitivity Analysis}
We have shown that the proposed \textsc{Spar-Sink} algorithm is not sensitive to the entropic regularization parameter $\eps$. In this section, we show that the robustness also holds for the marginal regularization parameter $\lambda$ in UOT problems.

We set $\lambda \in \{0.1, 1, 5\}$, with the remaining settings being the same as those in Section~\ref{sec:simu-subsec1}. 
The results are presented in Fig.~\ref{fig:simu-uot2}, which depicts the comparison of estimation errors among various methods w.r.t. the classical Sinkhorn algorithm, represented as $\operatorname{RMAE}^{(\UOT)}$, versus increasing subsample sizes.
We observe that \textsc{Spar-Sink} performs the best in all cases, and its estimation error becomes smaller as $\eta$ decreases, i.e., from \textbf{R1} to \textbf{R2} and \textbf{R3}. Such an observation indicates the proposed method can fully exploit the sparsity of the kernel matrix, leading to superior estimation accuracy.

\begin{figure}[!t]
    \centering
    \includegraphics[width=0.8\linewidth]{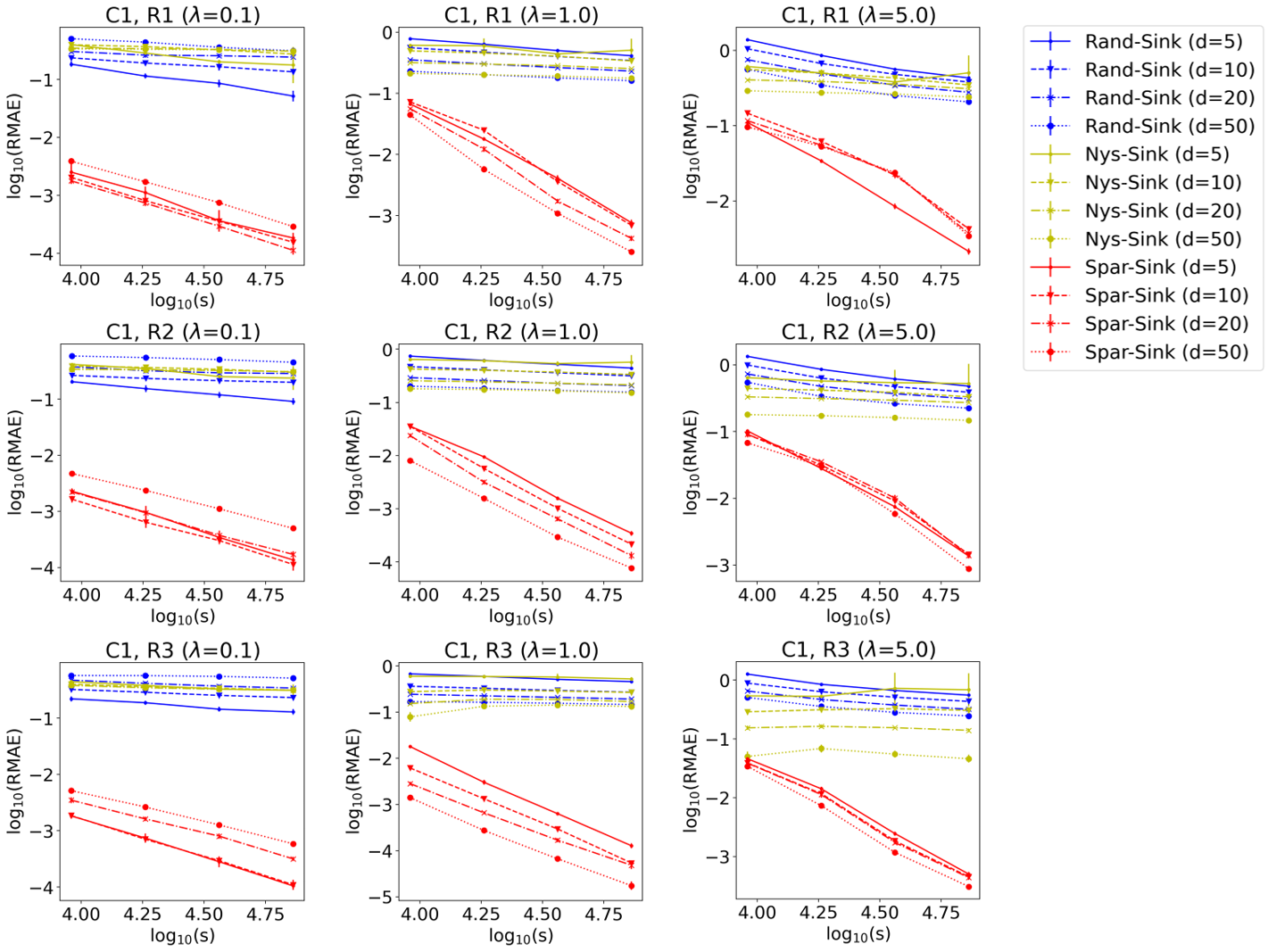}
    \caption{Comparison of subsampling-based methods w.r.t. $\operatorname{RMAE}^{(\UOT)}$ versus increasing $s$ (in log-log scale). Each row represents a different sparsity ratio (\textbf{R1---R3}), and each column represents a different $\lambda$. Different methods are marked by different colors, respectively, and each line type represents a different dimension $d$. Vertical bars are the standard errors.}
    \label{fig:simu-uot2}
\end{figure}

\subsection{Asymptotic Convergence}

To demonstrate the asymptotic convergence of our proposed method, we display the estimation error $\operatorname{RMAE}^{(\OT)}$ versus increasing sample sizes $n$ in Fig.~\ref{fig:simu-ot-n}, where $n\in\{2^0, 2^1, \ldots, 2^6\}\times 10^2$, $s = 8 s_0(n)$, and $\eps = 0.1$.
Other choices of $\eps$ yield similar results and thus are omitted here.
In Fig.~\ref{fig:simu-ot-n}, the \textsc{Spar-Sink} algorithm always yields a smaller estimation error than competitors. Notably, when $d$ is relatively large (e.g., $d \geq 10$), $\operatorname{RMAE}^{(\OT)}$ of \textsc{Spar-Sink} decreases substantially faster than that of \textsc{Nys-Sink}, which indicates a higher convergence rate.

\begin{figure}[!t]
    \centering
    \includegraphics[width=0.95\linewidth]{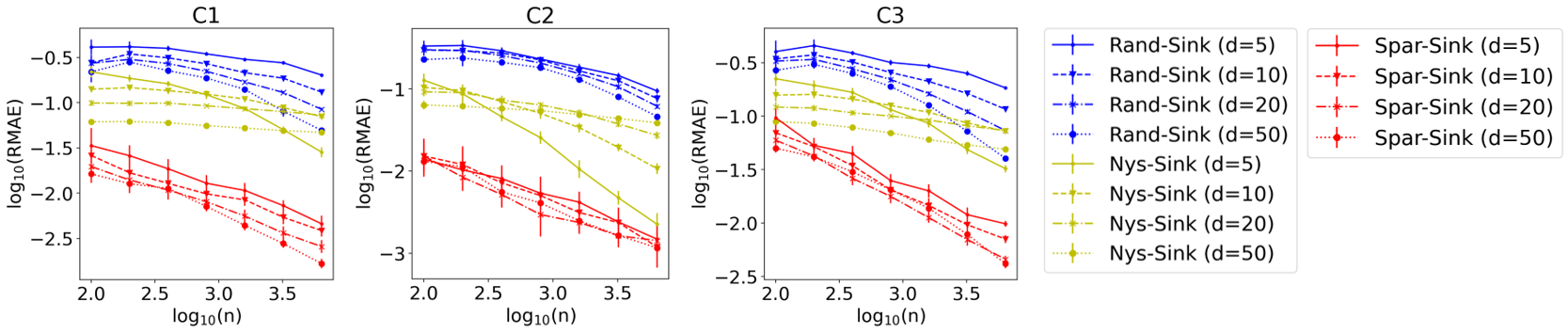}
    \caption{Comparison of different methods w.r.t. $\operatorname{RMAE}^{(\OT)}$ versus increasing $n$ (in log-log scale). Each column represents a different data generation pattern (\textbf{C1---C3}). Vertical bars are the standard errors.}
    \label{fig:simu-ot-n}
\end{figure}

Figure~\ref{fig:simu-uot-n} displays the results of $\operatorname{RMAE}^{(\UOT)}$ versus increasing $n$ and $s = 8 s_0(n)$, under $\eps = 0.1$ and $\lambda = 0.1$.
As shown in Fig.~\ref{fig:simu-uot-n}, the estimations of both \textsc{Rand-Sink} and \textsc{Nys-Sink} methods become worse as $n$ grows, while \textsc{Spar-Sink} converges significantly faster with the increase of $n$.

\begin{figure}[!t]
    \centering
    \includegraphics[width=0.95\linewidth]{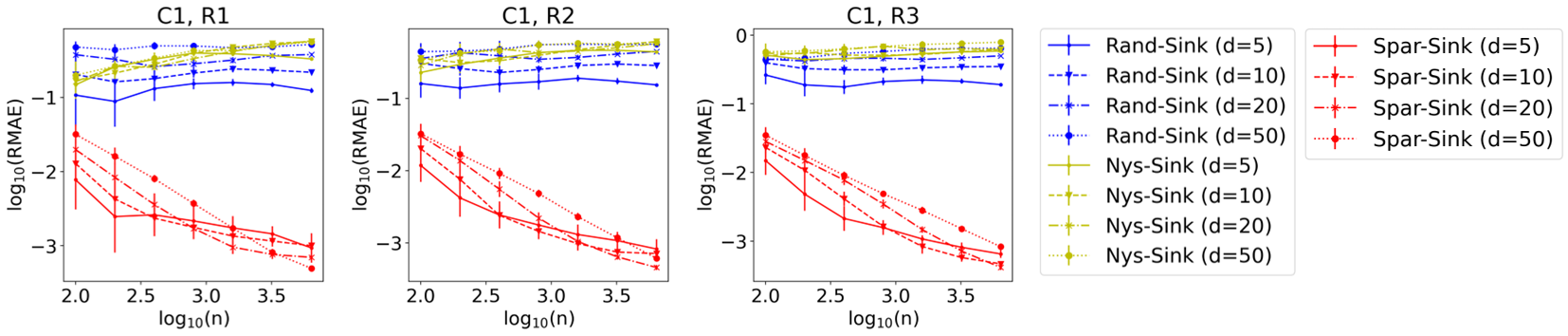}
    \caption{Comparison of different methods w.r.t. $\operatorname{RMAE}^{(\UOT)}$ versus increasing $n$ (in log-log scale). Each column represents a different sparsity ratio (\textbf{R1---R3}). Vertical bars are the standard errors.}
    \label{fig:simu-uot-n}
\end{figure}

\subsection{Approximation of Wasserstein Barycenters}

\textit{Experiments on synthetic data.} We compare the proposed \textsc{Spar-IBP} method with the classical IBP method, as well as two subsampling-based methods, \textsc{Nys-IBP} and \textsc{Rand-IBP}, which are direct extensions of \textsc{Nys-Sink} and \textsc{Rand-Sink} to approximate Wasserstein barycenters, respectively.

The input probability measures $\b_1,\b_2,\b_3\in\Delta^{n-1}$ are generated as:
\begin{itemize}
\item $\b_1$ is an empirical Gaussian distribution $N(\frac{1}{5},\frac{1}{50})$; 
\item $\b_2$ is an empirical Gaussian mixture distribution $\frac{1}{2} N(\frac{1}{2},\frac{1}{60}) + \frac{1}{2} N(\frac{4}{5},\frac{1}{80})$;
\item $\b_3$ is an empirical t-distribution with 5 degrees of freedom $t_5(\frac{3}{5},\frac{1}{100})$.
\end{itemize}
After generating the measures as above, we add $10^{-2}\max_{i\in[n]} b_{k,i}$ to each component of $\b_k$ and then normalize it such that $\sum_{i\in[n]} b_{k,i} = 1$, for $k\in[m]$.
Suppose the measures and their barycenter share the same support points $\{\x_i\}_{i=1}^n$, where $\x_i$'s are randomly and uniformly located over $(0,1)^d$, with $n=10^3$ and $d\in\{5,10,20\}$. Then, the cost matrices $\C_1=\C_2=\C_3$ are defined by squared Euclidean distances.
We set $\bm{w}=\mathbf{1}_m/m$, $\eps\in\{5,5^0,5^{-1}\}\times 10^{-2}$ and $s=\{5,10,15,20\}\times s_0(n)$ with $s_0(n) = 10^{-3} n \log^4(n)$.
For comparison, we calculate the approximation error of each estimator based on 100 replications, i.e.,
$$
\operatorname{Error} = \frac{1}{100}\sum_{i=1}^{100}\|\widetilde{\bm q}_\eps^{\ast(i)} - \bm q_\eps^{\ast(i)}\|_1,
$$
where $\widetilde{\bm q}_\eps^{\ast(i)}$ represents the estimator in the $i$th replication, and $\bm q_\eps^{\ast(i)}$ is obtained by the IBP algorithm. The results are shown in Fig.~\ref{fig:simu-bary}, from which we observe the proposed \textsc{Spar-IBP} method outperforms competitors in most circumstances, with its advantage becoming more prominent as the value of $\eps$ decreases. The comparison of CPU time has similar pattern to Fig.~\ref{fig:simu-time} and is omitted here.

\begin{figure}[!t]
    \centering
    \includegraphics[width=0.82\linewidth]{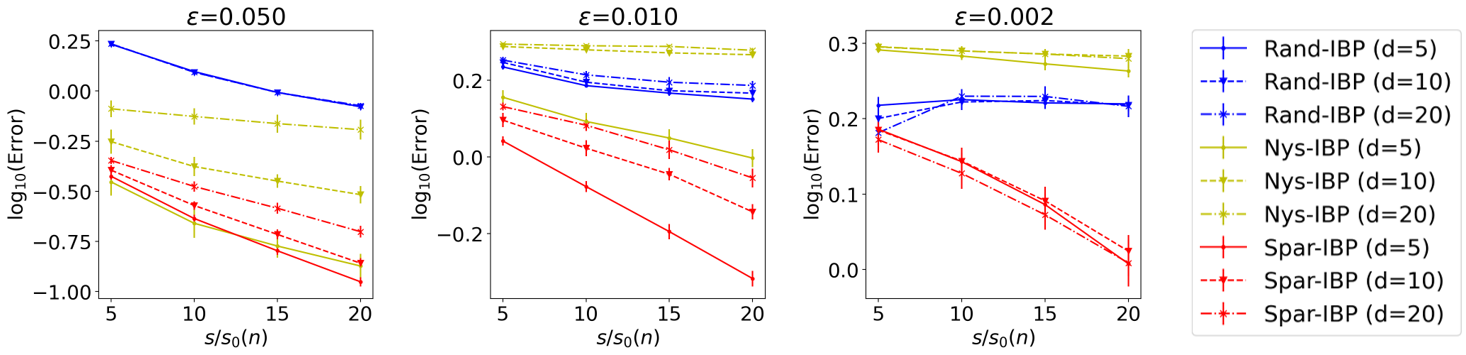}
    \caption{Comparison of different methods w.r.t. $\log_{10}(\operatorname{Error})$ versus increasing $s/s_0(n)$ under different levels of $\eps$. Each color marks a specific method, and each line type represents a different dimension $d$. Vertical bars are the standard errors.}
    \label{fig:simu-bary}
\end{figure}

\textit{Experiments on MNIST.} 
Further, we evaluate our \textsc{Spar-IBP} algorithm on the MNIST data set \citep{lecun1998gradient} following the work of \cite{cuturi2014fast}.
For each digit from 0 to 9, we randomly select 15 images from the data set, and uniformly rescale each image between half-size and double-size of its original scale at random. After that, each image is normalized such that all pixel values add up to 1. Then, the images are translated randomly within a $64\times 64$ grid, with a bias towards corners. 
Given the reshaped images with equal weights (i.e., $\bm{w}=\mathbf{1}_m/m$), we compute their Wasserstein barycenter. We also include the performance of the IBP algorithm for comparison. The images and results are shown in Fig.~\ref{fig:mnist-bary}. The regularization parameter is set to be $\eps=10^{-3}$ for both methods, and the subsample parameter is taken as $s=20s_0(n)$ for \textsc{Spar-IBP}.

\begin{figure}[!t]
    \centering
    \includegraphics[width=0.9\linewidth]{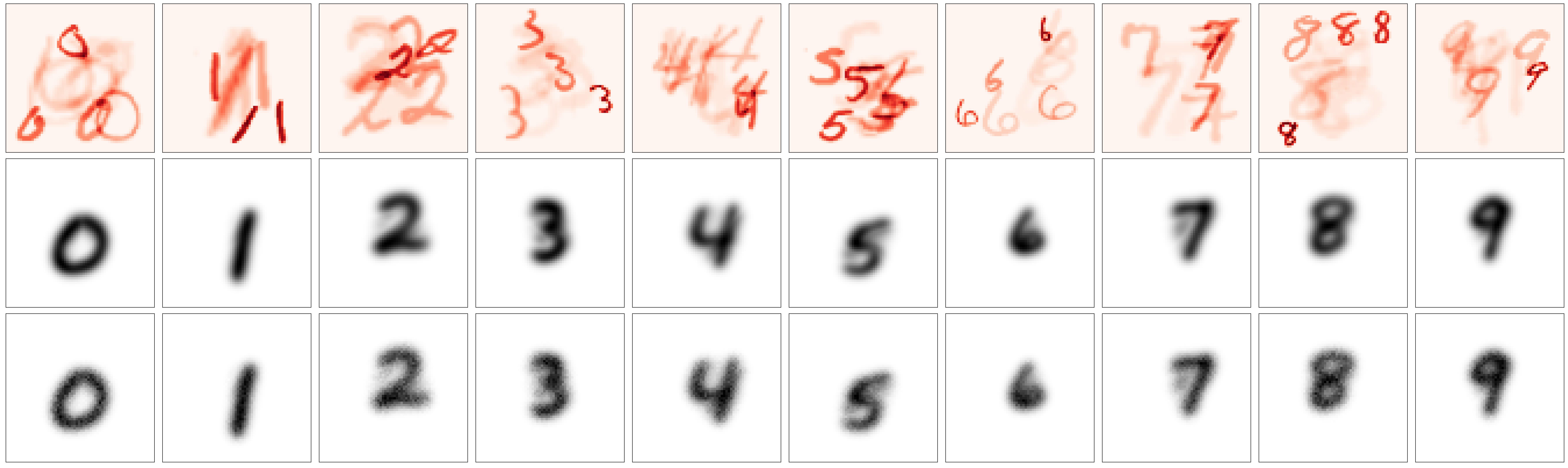}
    \caption{(Top row) For each digit, 8 out of the 15 rescaled and translated images are randomly chosen for illustration. (Middle row) Barycenters approximated by the \textsc{IBP} method. (Bottom row) Barycenters approximated by the \textsc{Spar-IBP} method.}
    \label{fig:mnist-bary}
\end{figure}

In Fig.~\ref{fig:mnist-bary}, we observe the approximate barycenters generated by \textsc{Spar-IBP} are almost as clear as those obtained by \textsc{IBP}.
Furthermore, with an average CPU time of 27.50s to compute one barycenter, \textsc{Spar-IBP} is considerably more efficient than \textsc{IBP}, which requires 340.29s.
Such results demonstrate the effectiveness and efficiency of our \textsc{Spar-IBP} method for approximating Wasserstein barycenters.

\section{Applications}
Following the recent work of \cite{le2021robust} and \cite{li2022hilbert}, we evaluate the performance of our proposed \textsc{Spar-Sink} method in two applications, color transfer and generative modeling.

\subsection{Color Transfer}
The objective is to transfer the color of an \textit{ocean sunset} image to an \textit{ocean daytime} image, as depicted in Fig.~\ref{fig:app-color1}.
The pixels of each image can be represented as point clouds in the three-dimensional RGB space.
Due to the large number of pixels in each image, which is nearly a million, we follow the preprocessing step in \cite{ferradans2014regularized} and \cite{le2021robust} to randomly downsample $n=5000$ pixels from each image, resulting in $\{\x_i\}_{i=1}^n, \{\y_j\}_{j=1}^n \subset \RR^{3}$ and use discrete uniform distributions to define $\a, \b \in \Delta^{n-1}$. We construct the cost matrix $\C$ using pairwise squared Euclidean distances, that is, $C_{ij}=\|\x_i-\y_j\|_2^2$.
To generate a new image with source content and target color, we compute the optimal transportation plan between $\a$ and $\b$ and extend the plan to the entire image using the nearest neighbor interpolation proposed by~\cite{ferradans2014regularized}.

\begin{figure}[!t]
    \centering
    \subfigure[Source and target images.]{
    \includegraphics[height=3.6cm]{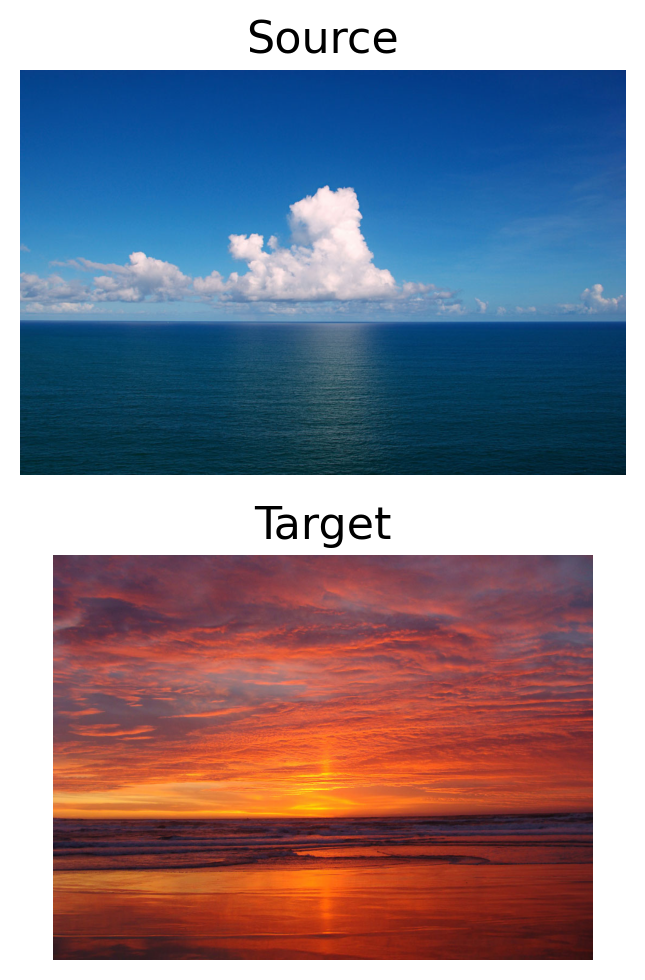}\label{fig:app-color1}
    }
    \hspace{0.2cm}
    \subfigure[Transferred images with source content and target color, respectively achieved by Sinkhorn, \textsc{Nys-Sink}, \textsc{Robust-NysSink}, and \textsc{Spar-Sink} (ours).]{
    \includegraphics[height=3.6cm]{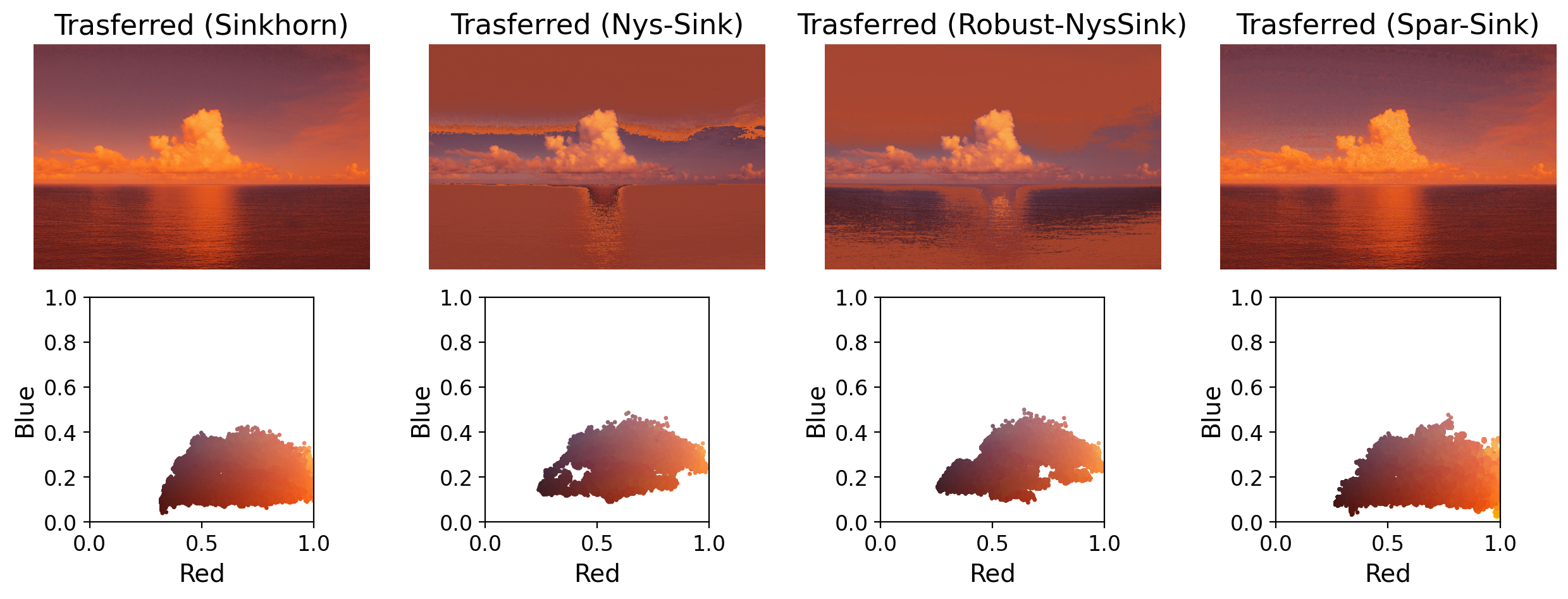}\label{fig:app-color2}
    }
    \caption{Comparison of different methods on color transfer. Subfigure~(b): (Top row) Transferred images. (Bottom row) The corresponding point clouds in RGB space. The color of the dot represents each point's RGB value.}
    \label{fig:app-color}
\end{figure}

To approximate the optimal transportation plan, we compare various methods, including the Sinkhorn, \textsc{Nys-Sink}, and \textsc{Spar-Sink} in classical OT formulations, as well as the \textsc{Robust-NysSink} from the robust OT framework \citep{le2021robust}. In this experiment, we set $\eps = 10^{-2}$ and $\lambda=10$, where $\lambda$ is the marginal regularized parameter in \textsc{Robust-NysSink}. For subsampling-based approaches, we take $s = 8s_0(n)$ for \textsc{Spar-Sink}, and $r = \lceil s/n \rceil$ for \textsc{Nys-Sink} and \textsc{Robust-NysSink}.

The results of the transferred images generated by these methods are presented in Fig.~\ref{fig:app-color2}.
Among the evaluated methods, we observe that \textsc{Spar-Sink} produces a transferred image that closely resembles the result of Sinkhorn, and its corresponding RGB scatter diagram is also similar to that of Sinkhorn.
In terms of computational efficiency, the CPU time of computing the plan is 60.45s (Sinkhorn), 12.92s (\textsc{Nys-Sink}), 27.74s (\textsc{Robust-NysSink}), and 3.15s (\textsc{Spar-Sink}), respectively.
These results demonstrate the effectiveness of our method on this common computer vision application.

\subsection{Generative Modeling}
In this section, we introduce a new variant of the Sinkhorn auto-encoder (SAE) \citep{patrini2020sinkhorn}, named \textsc{Spar-Sink} auto-encoder (SSAE), by using the proposed \textsc{Spar-Sink} approach. Specifically, we employ \textsc{Spar-Sink} to approximate the Sinkhorn divergence \citep{genevay2018learning, genevay2019sample, feydy2019interpolating} between the latent prior distribution and the expected posterior distribution during the training of auto-encoders. We assess the efficacy of this newly proposed generative model in several image generation tasks and compare it with the original SAE.

Assume $f$ (resp. $g$) is an encoder (resp. decoder) parameterized by a neural network, and $p_Z$ is a prior distribution on the latent space. Given a set of samples from a data distribution, i.e., $x_1, \ldots, x_N \sim p_X$, the objective of SAE is formulated as
$$
\min_{f,g}   \frac{1}{N} \sum_{i=1}^N c\left(x_i, g(f(x_i))\right) +\gamma S\left(f_{\#} p_X, p_Z\right),
$$
where $f_{\#}$ denotes the push-forward operator, $c(\cdot,\cdot)$ represents the reconstruction loss, and $S(\cdot,\cdot)$ is a regularizer with weight $\gamma > 0$ defined by Sinkhorn divergence, that is,
\begin{equation}\label{eq:sink-div}
    S\left(f_{\#} p_X, p_Z\right) = \OT_\eps \left(f_{\#} p_X, p_Z\right) -\frac{1}{2}\left( \OT_\eps \left(f_{\#} p_X, f_{\#} p_X\right) + \OT_\eps \left(p_Z, p_Z\right)\right).
\end{equation}
The objective of SSAE replaces $\OT_\eps(\cdot,\cdot)$ in~\eqref{eq:sink-div} with its approximation, $\widetilde{\OT}_\eps(\cdot,\cdot)$, computed using Algorithm~\ref{alg:core-ot}.

We train auto-encoders to embed the MNIST data \citep{lecun1998gradient} into a 10-dimensional latent space. The auto-encoding architecture is identical to that used in \cite{kolouri2019sliced}.
We use the Euclidean distance as the distance between samples, the binary cross entropy plus $\ell_1$ loss as the reconstruction loss, the standard Gaussian distribution as $p_Z$, and Adam \citep{kingma2014adam} as the optimizer.
For fairness, both SAE and SSAE employ the same hyperparameters: the regularization parameters are $\gamma=0.05$ and $\eps=0.01$; the number of epochs is 40; the batch size $n=500$; the learning rate is $0.001$; other parameters are set by default. Additionally, we set the subsample parameter $s=10s_0(n)$ for SSAE.

We compare SAE and SSAE w.r.t. Fr{\'e}chet inception distance (FID) \citep{heusel2017gans} between $10,000$ test samples and $10,000$ randomly generated samples, and also record their running time on an RTX 3090 GPU. The comparison is conducted based on 100 replications, and the results are presented in Table~\ref{tab:ae}, which shows that the proposed SSAE generator achieves a smaller FID in just half the time compared to SAE. 
Moreover, we provide image interpolation and reconstruction results obtained by SSAE in Fig.~\ref{fig:app-gen}, further highlighting the capability of our proposed method in generative modeling tasks.

\begin{table}[t]
\centering
\label{tab:ae}
\begin{tabular}{ccc}
\hline
Methods & FID & Time  \\
\hline
SAE & 24.72$_{\pm 0.13}$ & 125.87 \\
SSAE & 23.65$_{\pm 0.06}$ & 64.81 \\    
\hline
\end{tabular}
\caption{Comparisons on learning image generators w.r.t. average FID score (with standard deviations presented in footnotes) and running time (in seconds) of an epoch iteration.}
\end{table}

\begin{figure}[!t]
    \centering
    \subfigure[Digit interpolation.]{
    \includegraphics[height=3.6cm]{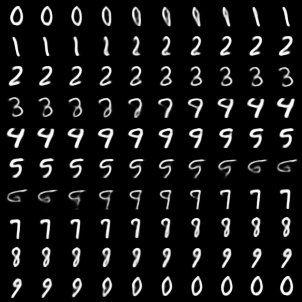}\label{fig:app-gen1}
    }
    \hspace{0.2cm}
    \subfigure[Digit reconstruction.]{
    \includegraphics[height=3.6cm]{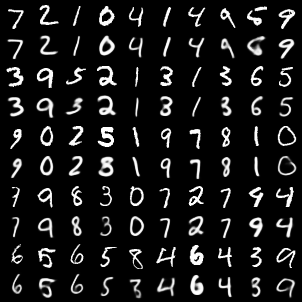}\label{fig:app-gen2}
    }
    \caption{The performance of SSAE on digit interpolation and reconstruction tasks. In the subfigure~(b), odd rows correspond to real images.}
    \label{fig:app-gen}
\end{figure}

\vskip 0.2in
\bibliography{ref} 

\end{document}